%% file: main.tex
\documentclass{article}                                                                                                                             [19/1902]

\usepackage{amsmath}
\usepackage{amsthm}
\usepackage{appendix}
\usepackage{bbold}
\usepackage{color}
\usepackage{environ}
\usepackage{etoolbox}
\usepackage{float}
\usepackage{fullpage}
\usepackage{graphicx}
\usepackage{hyperref}
\usepackage{mathtools}
\usepackage{natbib}
\usepackage{nicefrac}
\usepackage{parskip}
\usepackage{times}
\usepackage{titling}
\usepackage{url}
\usepackage{verbatim}
\usepackage{wrapfig}
\usepackage{xspace}

\usepackage{algorithm}
\usepackage{algorithmic}

\newif\ifdraft
\draftfalse

\newif\ifuglify
\uglifytrue

\input{helpers}

\title{A Light Touch for Heavily Constrained SGD}

\author{%
Andrew Cotter, Maya Gupta, Jan Pfeifer\\
Google Inc.\\
1600 Amphitheatre Parkway\\
Mountain View, CA 94043 \\
\texttt{\{acotter,mayagupta,janpf\}@google.com}
}

\predate{}
\date{}
\postdate{}

\begin{document}

\blfootnote{This version was also presented at the 29th Conference on Learning Theory (COLT 2016).}

\label{document:begin}

\maketitle

\togglefalse{istwocolumn}

\begin{abstract}%
\input{abstract}
\end{abstract}

\input{sec-introduction}
\input{sec-constraints}
\input{sec-algorithm}
\input{sec-comparison}
\input{sec-practical}
\input{sec-experiments}
\input{sec-conclusion}
\input{sec-acknowledgments}

\newpage
\clearpage

\bibliography{main}
\bibliographystyle{abbrvnat}

\label{document:middle}

\newpage
\clearpage
\appendix
\showproofstrue

\label{document:appendix}

\input{app-mirror}

\input{app-sgd}

\input{app-full-light}
\input{app-mid}
\label{document:end}

\end{document}

%% file: helpers.tex
\hypersetup{colorlinks = false, pdfborder = {0 0 0}}

\newtoggle{istwocolumn}

\renewcommand{\eqref}[1]{Equation~\ref{eq:#1}}
\newcommand{\secref}[1]{Section~\ref{sec:#1}}
\newcommand{\appref}[1]{Appendix~\ref{sec:#1}}
\newcommand{\figref}[1]{Figure~\ref{fig:#1}}
\newcommand{\tabref}[1]{Table~\ref{tab:#1}}
\newcommand{\algref}[1]{Algorithm~\ref{alg:#1}}
\newcommand{\thmref}[1]{Theorem~\ref{thm:#1}}
\newcommand{\lemref}[1]{Lemma~\ref{lem:#1}}
\newcommand{\corref}[1]{Corollary~\ref{cor:#1}}
\newcommand{\eqrefs}[2]{Equations~\ref{eq:#1} and~\ref{eq:#2}}

\newcommand{\algrefs}[2]{Algorithms~\ref{alg:#1} and~\ref{alg:#2}}
\newcommand{\thmrefs}[2]{Theorems~\ref{thm:#1} and~\ref{thm:#2}}
\newcommand{\lemrefs}[2]{Lemmas~\ref{lem:#1} and~\ref{lem:#2}}
\newcommand{\correfs}[2]{Corollaries~\ref{cor:#1} and~\ref{cor:#2}}

\ifdraft
  \newcommand{\TODO}[1]{\textcolor{red}{[TODO: #1]}\xspace}
  \newcommand{\NOTE}[1]{\textcolor{blue}{[NOTE: #1]}\xspace}
\else
  \newcommand{\TODO}[1]{\xspace}
  \newcommand{\NOTE}[1]{\xspace}
\fi

\ifuglify
  \newcommand{\eqcomma}{,}
  \newcommand{\eqperiod}{.}
\else
  \newcommand{\eqcomma}{}
  \newcommand{\eqperiod}{}
\fi

\newcommand{\youtube}{YouTube\xspace}

\makeatletter
\def\blfootnote{\gdef\@thefnmark{}\@footnotetext}
\makeatother

\makeatletter
\newtheorem*{rep@theorem}{\rep@title}
\newcommand{\newreptheorem}[2]{\newenvironment{rep#1}[1]{\def\rep@title{#2 \ref{##1}}\begin{rep@theorem}}{\end{rep@theorem}}}
\makeatother

\newtheorem{theorem}{Theorem}
\newtheorem{lemma}{Lemma}
\newtheorem{corollary}{Corollary}
\newreptheorem{theorem}{Theorem}
\newreptheorem{lemma}{Lemma}
\newreptheorem{corollary}{Corollary}

\newif\ifreptheorem
\reptheoremfalse
\newif\ifshowproofs
\showproofsfalse
\newcommand{\switchreptheorem}[2]{
  \ifreptheorem
    #1\xspace
  \else
    #2\xspace
  \fi
}
\newcommand{\switchshowproofs}[2]{
  \ifshowproofs
    #1\xspace
  \else
    #2\xspace
  \fi
}
\NewEnviron{thm}[1]{
  \makeatletter
  \ifcsname defined:thm:#1\endcsname
    \reptheoremtrue
    \begin{reptheorem}{thm:#1}\BODY\end{reptheorem}
    \reptheoremfalse
  \else
    \begin{theorem}\label{thm:#1}\BODY\end{theorem}
    \csgdef{defined:thm:#1}{}
  \fi
  \makeatother
}
\NewEnviron{lem}[1]{
  \makeatletter
  \ifcsname defined:lem:#1\endcsname
    \reptheoremtrue
    \begin{replemma}{lem:#1}\BODY\end{replemma}
    \reptheoremfalse
  \else
    \begin{lemma}\label{lem:#1}\BODY\end{lemma}
    \csgdef{defined:lem:#1}{}
  \fi
  \makeatother
}
\NewEnviron{cor}[1]{
  \makeatletter
  \ifcsname defined:cor:#1\endcsname
    \reptheoremtrue
    \begin{repcorollary}{cor:#1}\BODY\end{repcorollary}
    \reptheoremfalse
  \else
    \begin{corollary}\label{cor:#1}\BODY\end{corollary}
    \csgdef{defined:cor:#1}{}
  \fi
  \makeatother
}
\NewEnviron{prf}[1]{
  \ifshowproofs
    \begin{proof}\BODY\end{proof}\label{sec:prf:#1}
  \else
    \begin{proof}In \appref{prf:#1}.\end{proof}
  \fi
}

\renewcommand{\paragraph}[1]{\textbf{#1}\;\xspace}
\newcommand{\plotwidth}{0.48\textwidth}

\newcommand{\cplusplus}{C++\xspace}

\newcommand{\iid}{\emph{i.i.d.}\xspace}
\newcommand{\wrt}{w.r.t.\xspace}
\newcommand{\dgf}{d.g.f.\xspace}
\newcommand{\ie}{i.e.\xspace}
\newcommand{\eg}{e.g.\xspace}
\newcommand{\egcite}[1]{\citep[\eg][]{#1}}

\newcommand{\probability}{\mathrm{Pr}}
\newcommand{\expectation}{\mathbb{E}}

\newcommand{\filtration}{\mathcal{F}}
\newcommand{\bregman}{B}

\newcommand{\N}{\mathbb{N}}

\newcommand{\R}{\mathbb{R}}

\newcommand{\subjectto}{\mathrm{s.t.\;}}
\newcommand{\where}{\mathrm{where\;}}

\newcommand{\inner}[2]{\left\langle {#1}, {#2} \right\rangle}
\newcommand{\norm}[1]{\left\lVert {#1} \right\rVert}

\newcommand{\abs}[1]{\left\lvert {#1} \right\rvert}
\newcommand{\tallmid}{\mathrel{} \middle \vert \mathrel{}}

\DeclareMathOperator{\argmin}{argmin}

\newcommand{\codecomment}[1]{\;\;\;\;\;\;\;\;//~\textit{#1}}
\newcounter{lineno}
\newenvironment{pseudocode}{\setcounter{lineno}{0}\begin{tabbing}\textbf{mm}\=mm\=mm\=mm\=mm\=\kill}{\end{tabbing}}
\newcommand{\codename}{\>}
\newcommand{\codecommentline}[1]{\>\stepcounter{lineno}\textbf{\arabic{lineno}}\'\>//~\textit{#1}}
\newcommand{\codeline}{\>\stepcounter{lineno}\textbf{\arabic{lineno}}\'\>}

\newcommand{\fullname}{FullTouch\xspace}
\newcommand{\lightname}{LightTouch\xspace}
\newcommand{\midname}{MidTouch\xspace}
\newcommand{\orderingname}{OrderingProjection\xspace}
\newcommand{\fullalg}{\mbox{\hyperref[alg:full]{\fullname}}\xspace}
\newcommand{\midalg}{\mbox{\hyperref[alg:mid]{\midname}}\xspace}
\newcommand{\lightalg}{\mbox{\hyperref[alg:light]{\lightname}}\xspace}

\newcommand{\grad}{\nabla}
\newcommand{\subgrad}{\check{\grad}}
\newcommand{\supgrad}{\hat{\grad}}
\newcommand{\stochasticsubgrad}{\check{\Delta}}
\newcommand{\stochasticsupgrad}{\hat{\Delta}}
\newcommand{\memgrad}{\mu}
\newcommand{\objective}{h}
\newcommand{\relaxedobjective}{\tilde{h}}
\newcommand{\flipschitz}{L_f}
\newcommand{\mdgrad}{\stochasticsubgrad}
\newcommand{\mdnorm}[1]{\norm{#1}}
\newcommand{\mddualnorm}[1]{\norm{#1}_*}
\newcommand{\mdnormradiusstar}{R_{*}}
\newcommand{\mdgradbound}{G}
\newcommand{\wgrad}{\stochasticsubgrad_w}
\newcommand{\alphagrad}{\stochasticsupgrad_{\alpha}}
\newcommand{\wnorm}[1]{\norm{#1}_w}
\newcommand{\wdualnorm}[1]{\norm{#1}_{w*}}
\newcommand{\alphanorm}[1]{\norm{#1}_{\alpha}}
\newcommand{\alphadualnorm}[1]{\norm{#1}_{\alpha*}}
\newcommand{\wnormradiusstar}{R_{w *}}
\newcommand{\alphanormradiusstar}{R_{\alpha *}}
\newcommand{\fgrad}{\stochasticsubgrad}
\newcommand{\ggrad}{\stochasticsubgrad_g}
\newcommand{\pgrad}{\stochasticsupgrad_p}
\newcommand{\pnorm}[1]{\norm{#1}_p}
\newcommand{\pdualnorm}[1]{\norm{#1}_{p*}}
\newcommand{\wdiameter}{D_w}
\newcommand{\wgradbound}{G_w}
\newcommand{\fgradbound}{G_f}
\newcommand{\ggradbound}{G_g}
\newcommand{\pnormradiusstar}{R_{p *}}
\newcommand{\glipschitz}{L_g}
\newcommand{\fullbound}{U_F}
\newcommand{\lightbound}{U_L}

%% file: abstract.tex
Minimizing empirical risk subject to a set of constraints can be a useful
strategy for learning restricted classes of functions, such as monotonic
functions, submodular functions, classifiers that guarantee a certain class
label for some subset of examples, etc. However, these restrictions may result
in a very large number of constraints.
Projected stochastic gradient descent (SGD) is often the default choice for
large-scale optimization in machine learning, but requires a projection after
each update. For heavily-constrained objectives, we propose an efficient
extension of SGD that stays close to the feasible region while only applying
constraints probabilistically at each iteration.
Theoretical analysis shows a compelling trade-off between per-iteration work
and the number of iterations needed on problems with a large number of
constraints.
%
%

%% file: sec-introduction.tex
\section{Introduction}\label{sec:introduction}

Many machine learning problems can benefit from the addition of constraints.
For example, one can learn monotonic functions by adding appropriate
constraints to ensure or encourage positive derivatives
everywhere~\egcite{ArcherWa93,Sill98,SpougeWaWi03,DanielsVe10,GuptaCoPfVoCaMaMoEs16}.
Submodular functions can often be learned from noisy examples by imposing
constraints to ensure submodularity holds. Another example occurs when one
wishes to guarantee that a classifier will correctly label certain
``canonical'' examples, which can be enforced by constraining the function
values on those examples. See \citet{QuHu11} for some other examples of
constraints useful in machine learning.

However, these practical uses of constraints in machine learning are
impractical in that the number of constraints may be very large, and scale
poorly with the number of features $d$ or number of training samples $n$. In
this paper we propose a new strategy for tackling such heavily-constrained
problems, with guarantees and compelling convergence rates for large-scale
convex problems.

A standard approach for large-scale empirical risk minimization is projected
stochastic gradient descent~\egcite{Zinkevich03,NemirovskiJuLaSh09}. Each SGD
iteration is computationally cheap, and the algorithm converges quickly to a
solution good enough for machine learning needs. However, this algorithm
requires a projection onto the feasible region after each stochastic gradient
step, which can be prohibitively slow if there are many non-trivial
constraints, and is not easy to parallelize. Recently, Frank-Wolfe-style
algorithms~\egcite{HazanKa12,Jaggi13} have been proposed that remove the
projection, but require a constrained linear optimization at each iteration.

We propose a new strategy for large-scale constrained optimization that, like
\citet{MahdaviYaJiYi12}, moves the constraints into the objective and finds an
approximate solution of the resulting unconstrained problem, projecting the
(potentially-infeasible) result onto the constraints only once, at the end.
Their work focused on handling only one constraint, but as they noted, multiple
constraints $g_1(x) \le 0, g_2(x) \le 0, \dots, g_m(x) \le 0$ can be reduced to
one constraint by replacing the $m$ constraints with their maximum: $\max_i
g_i(x) \le 0$. However, this still requires that all $m$ constraints be checked
at every iteration. In this paper, we focus on the computational complexity as
a function of the number of constraints $m$, and show that it is possible to
achieve good convergence rates without checking constraints so often.

The key challenge to handling a large number of constraints is determining
which constraints are active at the optimum of the constrained problem, which
is likely to be only a small fraction of the total constraint set. For example,
for linear inequality constraints on a $d$-dimensional problem, no more than
$d$ of the constraints will be active at the optimum, and furthermore, once the
active constraints are known, the problem reduces to solving the unconstrained
problem that results from projecting onto them, which is typically vastly
easier.

To identify and focus on the important constraints, we propose learning a
probability distribution over the $m$ constraints that concentrates on the
most-violated, and \emph{sampling} constraints from this evolving distribution
at each iteration. We call this approach \lightalg because at each iteration
only a few constraints are checked, and the solution is only nudged toward the
feasible set. \lightalg is suitable for convex problems, but we also propose a
variant, \midalg, that enjoys a superior convergence rate on strongly convex
problems. These two algorithms are introduced and analyzed in
\secref{algorithm}.

Our proposed strategy removes the per-iteration $m$-dependence on the number of
constraint evaluations. \lightalg and \midalg do need more iterations to
converge, but each iteration is faster, resulting in a net performance
improvement. To be precise, we show that the total number of constraint checks
required to achieve $\epsilon$-suboptimality when optimizing a non-strongly
convex objective decreases from $O(m/\epsilon^2)$ to $\tilde{O}((\ln
m)/\epsilon^2 + m(\ln
m)^{\nicefrac{3}{2}}/\epsilon^{\nicefrac{3}{2}})$---notice that the
$m$-dependence of the dominant (in $\epsilon$) term has decreased from $m$ to
$\ln m$. For a $\lambda$-strongly convex objective, the dominant (again in
$\epsilon$) term in our bound on the number of constraint checks decreases from
$O(m/\lambda^2 \epsilon)$ to $\tilde{O}((\ln m)/\lambda^2\epsilon)$, but like
the non-strongly convex result this bound contains lower-order terms with worse
$m$-dependencies. A more careful comparison of the performance of our
algorithms can be found in \secref{comparison}.

While they check fewer than $m$ constraints per iteration, these algorithms do
need to pay a $O(m)$ per-iteration \emph{arithmetic} cost. When each constraint
is expensive to check, this cost can be neglected. However, when the constraints
are simple to check (\eg box constraints, or the lattice monotonicity
constraints considered in our experiments), it can be partially addressed by
transforming the problem into an equivalent one with fewer more costly
constraints. This, as well as other practical considerations, are discussed in
\secref{practical}.

Experiments on a large-scale real-world heavily-constrained ranking problem
show that our proposed approach works well in practice.
This problem was too large for a projected SGD implementation using an
off-the-shelf quadratic programming solver to perform projections, but was
amenable to an approach based on a fast approximate projection routine tailored
to this particular constraint set. Measured in terms of runtime, however,
\lightname was still significantly faster.
Each constraint in this problem is trivial, requiring only a single comparison
operation to check, so the aforementioned $O(m)$ arithmetic cost of \lightname
is a significant issue. Despite this, \lightname was roughly as fast as the
\citet{MahdaviYaJiYi12}-like algorithm \fullname. In light of other experiments
showing that \lightname checks dramatically fewer constraints in total than
\fullname, we believe that \lightname is well-suited to machine learning
problems with many nontrivial constraints.

%% file: sec-constraints.tex
\section{Heavily Constrained SGD}\label{sec:constraints}

\input{figures/tab-notation}
Consider the constrained optimization problem:
\begin{align}
  \label{eq:constrained-problem} \min_{w\in\mathcal{W}} & f\left(w\right) \\
  \notag \subjectto & g_i\left(w\right) \le 0 \;\; \forall i \in
  \left\{1,\dots,m\right\} \eqcomma
\end{align}
where $\mathcal{W} \subseteq \R^d$ is bounded, closed and convex, and $f :
\mathcal{W} \rightarrow \R$ and all $g_i : \mathcal{W} \rightarrow \R$ are
convex (our notation is summarized in \tabref{notation}). We assume that
$\mathcal{W}$ is a simple object, \eg an $\ell^2$ ball, onto which it is
inexpensive to project, and that the ``trickier'' aspects of the domain are
specified via the constraints $g_i(w) \le 0$.
Notice that we consider constraints written in terms of arbitrary convex
functions, and are not restricted to \eg only linear or quadratic constraints.


\subsection{\fullname: A Relaxation with a Feasible Minimizer}

We build on the approach of \citet{MahdaviYaJiYi12} to relax
\eqref{constrained-problem}. Defining $g(w) = \max_i g_i(w)$ and introducing a
Lagrange multiplier $\alpha$ yields the equivalent optimization problem:
\begin{equation}
  \label{eq:lagrangian-objective} \max_{\alpha \ge 0} \min_{w\in\mathcal{W}}
  f\left(w\right) + \alpha g\left(w\right) \eqperiod
\end{equation}
Directly optimizing over $w$ and $\alpha$ is problematic because the optimal
value for $\alpha$ is infinite for any $w$ that violates a constraint. Instead,
we follow \citet[Section 4.2]{MahdaviYaJiYi12} in relaxing the problem by
adding an upper bound of $\gamma$ on $\alpha$, and using the fact that
$\max_{0\le\alpha\le\gamma} \alpha g(w) = \gamma \max(0, g(w))$.

In the following lemma, we show that, with the proper choice of $\gamma$, any
minimizer of this relaxed objective is a feasible solution of
\eqref{constrained-problem}, indicating that using stochastic gradient descent
(SGD) to minimize the relaxation ($h(w)$ in the lemma below) will be effective.

\input{figures/alg-full}

\medskip
\input{theorems/lem-projection}
\medskip

The strategy of applying SGD to $h(w)$, detailed in \algref{full}, which we
call \fullname, has the same ``flavor'' as the algorithms proposed by
\citet{MahdaviYaJiYi12}, and we use it as a baseline comparison point for our
other algorithms.

Application of a standard SGD bound to \fullalg shows that it converges at a
rate with no explicit dependence on the number of constraints $m$, measured in
terms of the number of iterations required to achieve some desired
suboptimality (see \appref{full-light:full}), although the $\gamma$ parameter
can introduce an \emph{implicit} $d$ or $m$-dependence, depending on the
constraints (discussed in \secref{constraints:gamma}). The main drawback of
\fullalg is that each iteration is expensive, requiring the evaluation of all
$m$ constraints, since differentiation of $g$ requires first identifying the
most-violated. This is the key issue we tackle with the \lightalg algorithm
proposed in \secref{algorithm}.

%

\subsection{Constraint-Dependence of $\gamma$}\label{sec:constraints:gamma}

The conditions on \lemref{projection} were stated in terms of $g$, instead of
the individual $g_i$s, because it is difficult to provide suitable conditions
on the ``component'' constraints without accounting for their interactions.

For a point $w$ where two or more constraints intersect, the subdifferential of
$g(w)$ consists of all convex combinations of subgradients of the intersecting
constraints, with the consequence that even if each of the subgradients of the
$g_i(w)$s has norm at least $\rho'$, subgradients of $g(w)$ will generally have
norms smaller than $\rho'$. Exactly how much smaller depends on the particular
constraints under consideration. We illustrate this phenomenon with the
following examples, but note that, in practice, $\gamma$ should be chosen
experimentally for any particular problem, so the question of the $d$ and
$m$-dependence of $\gamma$ is mostly of theoretical interest.

\paragraph{Box Constraints} Consider the $m = 2d$ box constraints $g_i(w) =
-w_i - 1$ and $g_{i+d}(w) = w_i - 1$, all of which have gradients of norm $1$.
At most $d$ constraints can intersect (at a corner of the $[-1,1]^d$ box), all
of which are mutually orthogonal, so the norm of any convex combination of
their gradients is lower bounded by that of their average, $\rho = 1/\sqrt{d}$.
Hence, one should choose $\gamma > \sqrt{d} \, \flipschitz$.

As in the above example, $\gamma \propto \sqrt{\min(m,d)}$ will suffice when
the subgradients of intersecting constraints are at least orthogonal, and
$\gamma$ can be smaller if they always have positive inner products. However,
if subgradients of intersecting constraints tend to point in opposing
directions, then $\gamma$ may need to be much larger, as in our next example:

\paragraph{Ordering Constraints} Suppose the $m = d-1$ constraints order the
components of $w$ as $w_1 \le w_2 \le \cdots \le w_d$, for which $g_i(w) = (w_i
- w_{i+1}) / \sqrt{2}$, gradients of which again have norm $1$. All of these
constraints may be active simultaneously, in which case there is widespread
cancellation in the average gradient $(e_1 - e_{d}) / (m \sqrt{2})$, where
$e_i$ is the $i$th standard unit basis vector. The norm of this average
gradient is $\rho = 1 / m$, so we should choose $\gamma > (d-1) \flipschitz$.

In light of this example, one begins to wonder if a suitable $\gamma$ will
necessarily \emph{exist}---fortunately, the convexity of $g$ enables us to
prove a trivial bound as long as $g(v)$ is strictly negative for some
$v\in\mathcal{W}$:

\medskip
\input{theorems/lem-gamma}
\medskip

\paragraph{Linear Constraints} Consider the constraints $A w \preceq b$, with
each row of $A$ having unit norm, $b_{\min} = \min_i b_i > 0$, and
$\mathcal{W}$ being the $\ell^2$ ball of radius $r$.  It follows from
\lemref{gamma} that $\gamma > (2r / b_{\min}) \flipschitz$ suffices. Notice
that the earlier box constraint example satisfies these assumptions (with
$b_{\min} = 1$ and $r = \sqrt{d}$).

As the above examples illustrate, subgradients of $g$ will be large at the
boundary if subgradients of the $g_i$s are large, \emph{and} the constraints
intersect at sufficiently shallow angles that, representing boundary
subgradients of $g$ as convex combinations of subgradients of the $g_i$s, the
components reinforce each other, or at least do not cancel \emph{too} much.
This requirement is related to the linear regularity assumption introduced by
\citet{Bauschke96}, and considered recently by \citet{WangChLiGu15}.

%% file: figures/tab-notation.tex
\begin{table*}[t]

\caption{Key notation. \TODO{$\subgrad$ and $\supgrad$?}}

\label{tab:notation}

\begin{center}

\begin{tabular}{lll}
  \hline
  \textbf{Symbol} & \textbf{Description} & \textbf{Definition} \\
  \hline
  %
  $\mathcal{W}$ & Bounded, closed and convex domain & $\mathcal{W} \subseteq \R^d$ \\
  $\Delta^m$ & $m$-dimensional simplex & $\Delta^m = \{ p \in \R^m \mid p_i \ge 0 \wedge \sum_{i=1}^m p_i = 1 \}$ \\
  $d$ & Dimension of $\mathcal{W}$ & \\
  $m$ & Number of constraints & \\
  $f$ & Unconstrained objective function & $f : \mathcal{W} \rightarrow \R$ \\
  $g_i$ & Convex constraint functions & $g_i : \mathcal{W} \rightarrow \R$ \\
  $g$ & Combined constraint function & $g(w) = \max_i g_i(w)$ \\
  $\Pi_w$ & Projection onto $\mathcal{W}$ & $\Pi_w(w) = \argmin_{\{w' \in \mathcal{W}\}} \norm{w - w'}_2$ \\
  $\Pi_p$ & Projection onto $\Delta^m$ & $\Pi_p(p) = p / \norm{p}_1$ \\
  $\Pi_g$ & Projection onto constraints & $\Pi_g(w) = \argmin_{\{w' \in \mathcal{W} : g(w') \le 0\}} \norm{w - w'}_2$ \\
  $\rho$ & Boundary gradient magnitude & If $g(w) = 0$, then $\rho \le \norm{\subgrad}_2$ for all $\subgrad \in \partial g(w)$ \\
  $\gamma$ & Constraint scaling factor & $\gamma > \flipschitz / \rho$ \\
  $\objective$ & Objective function & $\objective(w) = f(w) + \gamma \max(0, g(w))$ \\
  $\relaxedobjective$ & Relaxed objective function & $\relaxedobjective(w, p) = f(w) + \gamma \sum_{i=1}^m p_i \max(0, g_i(w))$ \\
  $\flipschitz$ & Lipschitz constant of $f$ & $\flipschitz \norm{w - w'}_2 \ge \abs{f(w) - f(w')}$ \\
  $\glipschitz$ & Lipschitz constant of the $g_i$s & $\glipschitz \norm{w - w'}_2 \ge \abs{g_i(w) - g_i(w')}$ \\
  $\wdiameter$ & Bound ($\ge 1$) on diameter of $\mathcal{W}$ & $\wdiameter \ge \sup_{w,w'\in\mathcal{W}} \max\{1, \norm{w - w'}_2 \}$ \\
  $\fgradbound$ & Bound on stochastic subgradients of $f$ & $\fgradbound \ge \norm{\fgrad^{(t)}}_2$ \\
  $\ggradbound$ & Bound on stochastic subgradients of $g_i$s & $\ggradbound \ge \norm{\subgrad \max(0, g_i(w))}_2$ \\
  $\fgrad$ & Stochastic subgradient of $f$ & \\
  $\wgrad$ & Stochastic subgradient of $\relaxedobjective$ \wrt $w$ & \\
  $\pgrad$ & Stochastic supergradient of $\relaxedobjective$ \wrt $p$ & \\
  $\memgrad$ & \multicolumn{2}{l}{Remembered gradient coordinates~\citep{JohnsonZh13}} \\
  $k$ & Minibatch size in \lightalg's $p$-update & \\
  $\bar{w}$ & Average iterate & $\bar{w} = (\sum_{t=1}^T w^{(t)}) / T$ \\
  \hline
\end{tabular}

\end{center}

\end{table*}

%% file: figures/alg-full.tex
\begin{algorithm*}[t]

\begin{pseudocode}
\codename \textbf{Hyperparameters:} $T$, $\eta$ \\
\codeline Initialize $w^{(1)} \in \mathcal{W}$ arbitrarily\\
\codeline For $t = 1$ to $T$:\\
\codeline \>Sample $\fgrad^{(t)}$ \codecomment{stochastic subgradient of $f(w^{(t)})$}\\
\codeline \>Let $\wgrad^{(t)} = \fgrad^{(t)} + \gamma \subgrad \max\{0, g(w^{(t)})\}$\\
\codeline \>Update $w^{(t+1)} = \Pi_w( w^{(t)} - \eta \wgrad^{(t)} )$ \codecomment{$\Pi_w$ projects its argument onto $\mathcal{W}$ \wrt $\norm{\cdot}_2$} \\
\codeline Average $\bar{w} = (\sum_{t=1}^T w^{(t)}) / T$\\
\codeline Return $\Pi_g(\bar{w})$ \codecomment{optional if small constraint violations are acceptable}
\end{pseudocode}

\caption{
  \textbf{(\fullname)} Minimizes $f$ on $\mathcal{W}$ subject to the single
  constraint $g(w) \le 0$. For problems with $m$ constraints $g_i(w) \le 0$,
  let $g(w) = \max_i g_i(w)$, in which case differentiating $\max\{0,g(w)\}$
  (line $4$) requires evaluating all $m$ constraints. This algorithm---our
  starting point---is similar to those proposed by \citet{MahdaviYaJiYi12},
  and like their algorithms only contains a single projection, at the end,
  projecting the potentially-infeasible result vector $\bar{w}$.
}

\label{alg:full}

\end{algorithm*}

%% file: theorems/lem-projection.tex
\begin{lem}{projection}
  \switchreptheorem{In the setting of \secref{constraints}, suppose}{Suppose}
  that $f$ is $\flipschitz$-Lipschitz, \ie $\abs{f(w) - f(w')} \le \flipschitz
  \norm{w - w'}_2$ for all $w,w' \in \mathcal{W}$, and that there is a constant
  $\rho > 0$ such that if $g(w) = 0$ then $\norm{\subgrad}_2 \ge \rho$ for all
  $\subgrad \in \partial g(w)$, where $\partial g(w)$ is the subdifferential of
  $g(w)$.

  For a parameter $\gamma > 0$, define:
  %
  \begin{equation*}
    \objective\left( w \right) = f\left(w\right) + \gamma \max\left\{ 0,
    g\left(w\right) \right\} \eqperiod
  \end{equation*}
  If $\gamma > \flipschitz / \rho$, then for any infeasible $w$ (\ie for which
  $g(w) > 0$):
  \begin{equation*}
    \objective\left(w\right) > \objective\left(\Pi_{g}\left(w\right)\right) =
    f\left(\Pi_{g}\left(w\right)\right)
    \;\;\;\;\;\; \mbox{and} \;\;\;\;\;\;
    \norm{w - \Pi_{g}\left(w\right)}_2 \le \frac{h\left(w\right) -
    h\left(\Pi_{g}\left(w\right)\right)}{\gamma \rho - \flipschitz} \eqcomma
  \end{equation*}
  where $\Pi_{g}\left(w\right)$ is the projection of $w$ onto the set $\{w \in
  \mathcal{W} : g(w) \le 0\}$ \wrt the Euclidean norm.
  %
\end{lem}
\begin{prf}{projection}
  Let $w \in \mathcal{W}$ be an arbitrary infeasible point. Because $f$ is
  $\flipschitz$-Lipschitz:
  \begin{equation}
    \label{eq:projection-objective} f\left(w\right) \ge
    f\left(\Pi_{g}\left(w\right)\right) - \flipschitz\norm{w -
    \Pi_{g}\left(w\right)}_2 \eqperiod
  \end{equation}
  Since $\Pi_{g}(w)$ is the projection of $w$ onto the constraints \wrt the
  Euclidean norm, we must have by the first order optimality conditions that
  there exists a $\nu \ge 0$ such that:
  \begin{equation*}
    0 \in \partial \norm{w - \Pi_{g}\left(w\right)}_2^2 + \nu \partial g\left(
    \Pi_{g}\left(w\right) \right) \eqperiod
  \end{equation*}
  This implies that $w - \Pi_{g}(w)$ is a scalar multiple of some $\subgrad \in
  \partial g(\Pi_{g}(w))$. Because $g$ is convex and $\Pi_{g}\left(w\right)$ is
  on the boundary, $g(w) \ge g(\Pi_{g}(w)) + \inner{\subgrad}{w - \Pi_{g}(w)} =
  \inner{\subgrad}{w - \Pi_{g}(w)}$, so:
  \begin{equation}
    \label{eq:projection-constraints} g(w) \ge \rho \norm{w - \Pi_{g}(w)}_2 \eqperiod
  \end{equation}
  Combining the definition of $\objective$ with
  \eqrefs{projection-objective}{projection-constraints} yields:
  \begin{equation*}
    \objective\left(w\right) \ge f\left(\Pi_{g}\left(w\right)\right) +
    \left(\gamma \rho - \flipschitz\right) \norm{w - \Pi_{g}(w)}_2 \eqperiod
  \end{equation*}
  Both claims follow immediately if $\gamma \rho > \flipschitz$.
\end{prf}

%% file: theorems/lem-gamma.tex
\begin{lem}{gamma}
  \switchreptheorem{In the setting of \secref{constraints}, suppose}{Suppose}
  that there exists a $v \in \mathcal{W}$ for which $g(v) < 0$, and let
  $\wdiameter \ge \sup_{w,w'\in\mathcal{W}} \norm{w - w'}_2$ bound the diameter
  of $\mathcal{W}$. Then $\rho = -g(v) / \wdiameter$ satisfies the conditions
  of \lemref{projection}.
\end{lem}
\begin{proof}
  Let $w \in \mathcal{W}$ be a point for which $g(w) = 0$, and $\subgrad \in
  \partial g(w)$ an arbitrary subgradient. By convexity, $g(v) \ge g(w) +
  \inner{v - w}{\subgrad}$. The Cauchy-Schwarz inequality then gives that:
  \begin{equation*}
    g(v) \ge -\norm{v - w}_2 \norm{\subgrad}_2 \eqcomma
  \end{equation*}
  from which the claim follows immediately.
\end{proof}

%% file: sec-algorithm.tex
\section{A Light Touch}\label{sec:algorithm}

This section presents the main contribution of this paper: an algorithm that
stochastically samples a small subset of the $m$ constraints at each SGD
iteration, updates the parameters based on the subgradients of the sampled
constraints, and carefully learns the distribution over the constraints to
produce a net performance gain.

We first motivate the approach by considering an oracle, then explain the
algorithm and present convergence results for the convex
(\secref{algorithm:light}) and strongly convex (\secref{algorithm:mid}) cases.

\subsection{Wanted: An Oracle For the Most Violated Constraint}

Because \fullalg only needs to differentiate the most violated constraint at
each iteration, it follows that if one had access to an oracle that identified
the most-violated constraint, then the overall convergence rate (including the
cost of each iteration) could \emph{only} depend on $m$ through $\gamma$. This
motivates us to \emph{learn} to predict the most-violated constraint, ideally
at a significantly better than linear-in-$m$ rate.


%
To this end, we further relax the problem of minimizing $h(w)$ (defined in
\lemref{projection}) by replacing $\gamma \max(0, g(w))$ with maximization over
a probability distribution (as in \citet{ClarksonHaWo10}),
yielding the equivalent convex-linear optimization problem:
\begin{align}
  \label{eq:relaxed-problem} \max_{p \in \Delta^m} \min_{w\in\mathcal{W}} &
  \relaxedobjective\left(w, p\right) \\
  \notag \where & \relaxedobjective\left(w, p\right) = f\left(w\right) + \gamma
  \sum_{i=1}^m p_i \max\left\{0, g_i\left(w\right)\right\} \eqperiod
\end{align}
Here, $\Delta^m$ is the $m$-dimensional simplex.
%
%
%
%
We propose optimizing over $w$ and $p$ jointly, thereby learning the
most-violated constraint, represented by the multinoulli distribution $p$ over
constraint indices, at the same time as we optimize over $w$.

\subsection{\lightname: Stochastic Constraint Handling}\label{sec:algorithm:light}

\input{figures/alg-light}
To optimize \eqref{relaxed-problem}, our proposed algorithm (\algref{light},
\lightname) iteratively samples stochastic gradients $\wgrad^{(t)}$ \wrt $w$
and $\pgrad^{(t)}$ \wrt $p$ of $\relaxedobjective(w,p)$, and then takes an SGD
step on $w$ and a multiplicative step on $p$:
\begin{equation*}
  w^{(t+1)} = \Pi_w\left( w^{(t)} - \eta \wgrad^{(t)} \right)
  \;\;\;\;\;\; \mbox{and} \;\;\;\;\;\;
  p^{(t+1)} = \Pi_p\left(\exp\left( \ln p^{(t)} + \eta \pgrad^{(t)}
  \right)\right) \eqcomma
\end{equation*}
where the $\exp$ and $\ln$ of the $p$-update are performed element-wise,
$\Pi_w$ projects onto $\mathcal{W}$ \wrt the Euclidean norm, and $\Pi_p$ onto
$\Delta^m$ via normalization (i.e. dividing its parameter by its sum).

The key to getting a good convergence rate for this algorithm is to choose
$\wgrad$ and $\pgrad$ such that they are both inexpensive to compute, and tend
to have small norms. For $\wgrad$, this can be accomplished straightforwardly,
by sampling a constraint index $i$ according to $p$, and taking:
\begin{equation*}
  \wgrad = \fgrad + \gamma \subgrad \max\left\{0, g_i\left(w\right)\right\}
  \eqcomma
\end{equation*}
where $\fgrad$ is a stochastic subgradient of $f$ and $\subgrad \max(0,
g_i(w))$ is a subgradient of $\max(0, g_i(w))$. Calculating each such $\wgrad$
requires differentiating only one constraint, and it is easy to verify that
$\wgrad$ is a subgradient of $\relaxedobjective$ \wrt $w$ in expectation over
$\fgrad$ and $i$. Taking $\fgradbound$ to be a bound on the norm of $\fgrad$
and $\ggradbound$ on the norms of subgradients of the $g_i$s shows that
$\wgrad$'s norm is bounded by $\fgradbound + \gamma \ggradbound$.

For $\pgrad$, some care must be taken. Simply sampling a constraint index $j$
uniformly and defining:
\begin{equation*}
  \pgrad = \gamma m e_j \max\left\{0,
  g_j\left(w\right)\right\} \eqcomma
\end{equation*}
where $e_j$ is the $j$th $m$-dimensional standard unit basis vector, does
produce a $\pgrad$ that in expectation is the gradient of $\relaxedobjective$
\wrt $p$, but it has a norm bound proportional to $m$. Such potentially large
stochastic gradients would result in the number of iterations required to
achieve some target suboptimality being proportional to $m^2$ in our final bound.

A typical approach to reducing the variance (and hence the expected magnitude)
of $\pgrad$ is minibatching: instead of sampling a single constraint index $j$
at every iteration, we could instead sample a subset $S$ of size $\abs{S} = k$
without replacement, and use:
\begin{equation*}
  \pgrad = \frac{\gamma m}{k} \sum_{j\in S} e_j \max\left\{0,
  g_j\left(w\right)\right\} \eqperiod
\end{equation*}
This is effective, but not enough, because reducing the variance by a factor of
$k$ via minibatching requires that we check $k$ times more constraints. For
this reason, in addition to minibatching, we center the stochastic gradients,
as is done by the well-known SVRG algorithm~\citep{JohnsonZh13}, by storing a
gradient estimate $\gamma \memgrad$ with $\memgrad \in \R^m$, at each iteration
sampling a set $S$ of size $\abs{S} = k$ uniformly without replacement, and
computing:
\begin{equation}
  \label{eq:johnson-zhang} \pgrad = \gamma \mu + \frac{\gamma m}{k} \sum_{j\in
  S} e_j \left(\max\left\{0, g_j(w)\right\} - \mu_j\right) \eqperiod
\end{equation}
We then update the $j$th coordinate of $\mu$ to be $\mu_j = \max\left\{0,
g_j(w)\right\}$ for every $j \in S$. The norms of the resulting stochastic
gradients will be small if $\gamma \mu$ is a good estimate of the gradient, \ie
$\mu_j \approx \max(0, g_j(w))$.

The difference between $\mu_j$ and $\max(0, g_j(w))$ can be bounded in terms of
how many consecutive iterations may have elapsed since $\mu_j$ was last
updated. It turns out (see \lemref{coupon} in \appref{full-light:light}) that
this quantity can be bounded uniformly by $O((m/k) \ln(m T))$ with high
probability, which implies that if the $g_i$s are $\glipschitz$-Lipschitz, then
$\abs{g_j(w) - \mu_j} \le \glipschitz \eta (\fgradbound + \gamma \ggradbound)
O((m/k) \ln(m T))$, since at most $O((m/k) \ln(m T))$ updates of magnitude
$\eta (\fgradbound + \gamma \ggradbound)$ may have occurred since $\mu_j$ was
last updated. Choosing $\eta \propto 1/\sqrt{T}$, as is standard, moves this
portion (the ``variance portion'') of the $\pgrad$-dependence out of the
dominant $O(1/\sqrt{T})$ term and into a subordinate term in our final bound.

The remainder of the $\pgrad$-dependence (the ``mean portion'') depends on the
norm of $\expectation[\pgrad] = \gamma \sum_j e_j \max(0, g_j(w))$. It is here
that our use of multiplicative $p$-updates becomes significant, because with
such updates the relevant norm is the $\ell^{\infty}$ norm, instead of \eg the
$\ell^2$ norm (as would be the case if we updated $p$ using SGD), thus we can
bound $\norm{\expectation[\pgrad]}_\infty$ with no explicit $m$-dependence.

The following theorem on the convergence rate of \lightalg is proved by
applying a mirror descent bound for saddle point problems while bounding the
stochastic gradient norms as described above.

\medskip
\input{theorems/thm-light}
\medskip

The most important thing to notice about this theorem is that the dominant
terms in the bounds on the number of iterations and number of constraint checks
are roughly $\gamma^2 \ln m$ times the usual $1/\epsilon^2$ convergence rate
for SGD on a non-strongly convex objective. The lower-order terms have a worse
$m$-dependence, however, with the result that, as the desired suboptimality
$\epsilon$ shrinks, the algorithm performs fewer constraint checks per
iteration until ultimately (once $\epsilon$ is on the order of $1/m^2$) only a
constant number are checked during each iteration.



\subsection{\midname: Strong Convexity}\label{sec:algorithm:mid}

\input{figures/alg-mid}
To this point, we have only required that the objective function $f$ be convex.
However, roughly the same approach also works when $f$ is taken to be
$\lambda$-strongly convex, although we have only succeeded in proving an
in-expectation result, and the algorithm, \algref{mid} (\midalg), differs
from \lightalg not only in that the $w$ updates use a $1/\lambda t$ step size,
but also in being a two-phase algorithm, the first of which, like \fullalg,
checks every constraint at each iteration, and the second of which, like
\lightalg with $k=1$, checks only two. The following theorem bounds the
convergence rate if we perform $T_1 \approx m \tau^2$ iterations in the first
phase and $T_2 \approx \tau^3$ in the second, where the parameter $\tau$
determines the total number of iterations performed:

\medskip
\input{theorems/thm-mid}
\medskip

Notice that the above theorem bounds not the suboptimality of $\Pi_g(\bar{w})$,
but rather its squared Euclidean distance from $w^*$, for which reason the
denominator of the highest order term depends on $\lambda^2$ rather than
$\lambda$.
Like \thmref{light} in the non-strongly convex case, the dominant terms above,
both in terms of the total number of iterations and number of constraint
checks, match the usual $1/\epsilon$ convergence rate for unconstrained
strongly-convex SGD with an additional $\gamma^2 \ln m$ factor, while the
lower-order terms have a worse $m$-dependence. As before, fewer constraint
checks will be performed per iteration as $\epsilon$ shrinks, reaching a
constant number (on average) once $\epsilon$ is on the order of $1/m^6$.

%% file: figures/alg-light.tex
\begin{algorithm*}[t]

\begin{pseudocode}
\codename \textbf{Hyperparameters:} $T$, $\eta$, $k$ \\
\codeline Initialize $w^{(1)} \in \mathcal{W}$ arbitrarily\\
\codeline Initialize $p^{(1)} \in \Delta^m$ to the uniform distribution\\
\codeline Initialize $\memgrad^{(1)}_j = \max\{0, g_j(w^{(1)})\}$ \codecomment{0 if $w^{(1)}$ is feasible}\\
\codeline For $t = 1$ to $T$:\\
\codeline \>Sample $\fgrad^{(t)}$ \codecomment{stochastic subgradient of $f(w^{(t)})$}\\
\codeline \>Sample $i^{(t)} \sim p^{(t)}$\\
\codeline \>Let $\wgrad^{(t)} = \fgrad^{(t)} + \gamma \subgrad \max\{0, g_{i^{(t)}}(w^{(t)})\}$\\
\codeline \>Update $w^{(t+1)} = \Pi_w( w^{(t)} - \eta \wgrad^{(t)} )$ \codecomment{$\Pi_w$ projects its argument onto $\mathcal{W}$ \wrt $\norm{\cdot}_2$}\\
\codeline \>Sample $S^{(t)} \subseteq \{1,\dots,m\}$ with $\abs{S^{(t)}} = k$ uniformly without replacement\\
\codeline \>Let $\pgrad^{(t)} = \gamma \memgrad^{(t)} + (\gamma m / k) \sum_{j \in S^{(t)}} e_j ( \max\{0, g_j(w^{(t)})\} -  \memgrad_j^{(t)} )$\\ 
\codeline \>Let $\memgrad^{(t+1)}_j = \max\{0, g_j(w^{(t)})\}$ if $j \in S^{(t)}$, otherwise $\memgrad^{(t+1)}_j = \memgrad^{(t)}_j$\\
\codeline \>Update $\tilde{p}^{(t+1)} = \exp( \ln p^{(t)} + \eta \pgrad^{(t)} )$ \codecomment{element-wise $\exp$ and $\ln$}\\
\codeline \>Project $p^{(t+1)} = \tilde{p}^{(t+1)} / \norm{\tilde{p}^{(t+1)}}_1$\\
\codeline Average $\bar{w} = (\sum_{t=1}^T w^{(t)}) / T$\\
\codeline Return $\Pi_g(\bar{w})$ \codecomment{optional if small constraint violations are acceptable}
\end{pseudocode}

\caption{
  \textbf{(\lightname)} Minimizes $f$ on $\mathcal{W}$ subject to the
  constraints $g_i(w) \le 0$ for $i\in\{1,\dots,m\}$. The algorithm learns an
  auxiliary probability distribution $p$ (lines 9--13) estimating how likely
  it is that each constraint is the most-violated. We assume that $k \le m$: if
  $k > m$, then the user is willing to check $m$ constraints per iteration
  \emph{anyway}, so \fullalg is the better choice.
  Like \fullalg, this algorithm finds a potentially-infeasible solution
  $\bar{w}$ which is only projected onto the feasible region at the end.
  %
  %
  Notice that while the $p$-update checks only $k$ constraints, it does require
  $O(m)$ arithmetic operations. This issue is discussed further in
  \secref{practical:aggregation}.
}

\label{alg:light}

\end{algorithm*}

%% file: theorems/thm-light.tex
\begin{thm}{light}
  \switchshowproofs{
    Suppose that the conditions of \lemrefs{projection}{light-suboptimality}
    apply.
  }{
    Suppose that the conditions of \lemref{projection} apply, with $g(w) =
    \max_i(g_i(w))$. Define $\wdiameter \ge \max\{1, \norm{w - w'}_2\}$ as a
    bound on the diameter of $\mathcal{W}$ (notice that we also choose
    $\wdiameter$ to be at least $1$), $\fgradbound \ge \norm{\fgrad^{(t)}}_2$
    and $\ggradbound \ge \norm{\subgrad \max(0, g_i(w))}_2$ as uniform upper
    bounds on the (stochastic) gradient magnitudes of $f$ and the $g_i$s,
    respectively, for all $i \in \{1,\dots,m\}$ and $w,w'\in\mathcal{W}$.
    We also assume that all $g_i$s are $\glipschitz$-Lipschitz \wrt
    $\norm{\cdot}_2$, \ie $\abs{g_i(w) - g_i(w')} \le \glipschitz \norm{w -
    w'}_2$.
  }
  Our result will be expressed in terms of a total iteration count
  $T_{\epsilon}$ satisfying:
  \begin{equation*}
    T_{\epsilon} =
    O\left( \frac{\left(\ln m\right) \wdiameter^2 \left(\fgradbound + \gamma
    \ggradbound + \gamma \glipschitz \wdiameter\right)^2
    \ln\frac{1}{\delta}}{\epsilon^2} \right) \eqperiod
  \end{equation*}
  \switchshowproofs{
    Define $k$ in terms of $T_{\epsilon}$ as in \lemref{light-suboptimality}.
    If $k \le m$, then we optimize \eqref{constrained-problem} using
    $T_{\epsilon}$ iterations of \algref{light} (\lightname) with $\eta$ as in
    \lemref{light-suboptimality}. If $k > m$, then we use $T_{\epsilon}$
    iterations of \algref{full} (\fullname) with $\eta$ as in
    \lemref{full-suboptimality}.
  }{
    Define:
    \begin{equation*}
      k = \left\lceil \frac{m \left(1 + \ln m\right)^{\nicefrac{3}{4}} \sqrt{1
      + \ln\frac{1}{\delta}} \sqrt{1 + \ln
      T_{\epsilon}}}{T_{\epsilon}^{\nicefrac{1}{4}}} \right\rceil
      \eqperiod
    \end{equation*}
    If $k \le m$, then we optimize \eqref{constrained-problem} using
    $T_{\epsilon}$ iterations of \algref{light} (\lightname), basing the
    stochastic gradients \wrt $p$ on $k$ constraints at each iteration, and
    using the step size:
    \begin{equation*}
      \eta = \frac{\sqrt{1 + \ln m} \wdiameter}{\left( \fgradbound + \gamma
      \ggradbound + \gamma \glipschitz \wdiameter\right) \sqrt{T_{\epsilon}}}
      \eqperiod
    \end{equation*}
    If $k > m$, then \lightalg would check more than $m$ constraints per
    iteration anyway, so we instead use $T_{\epsilon}$ iterations of
    \algref{full} (\fullname) with the step size:
    \begin{equation*}
      \eta = \frac{\wdiameter}{\left(\fgradbound + \gamma \ggradbound\right)
      \sqrt{T_{\epsilon}}} \eqperiod
    \end{equation*}
  }
  In either case, we perform $T_{\epsilon}$ iterations, requiring a total of
  $C_{\epsilon}$ ``constraint checks'' (evaluations or differentiations of a
  single $g_i$):
  \begin{align*}
    C_{\epsilon} = &
    \tilde{O}\left(
    \frac{\left(\ln m\right) \wdiameter^2 \left(\fgradbound + \gamma
    \ggradbound + \gamma \glipschitz \wdiameter\right)^2
    \ln\frac{1}{\delta}}{\epsilon^2} \right. \\
    & \left.  + \frac{m \left(\ln m\right)^{\nicefrac{3}{2}}
    \wdiameter^{\nicefrac{3}{2}} \left(\fgradbound + \gamma \ggradbound +
    \gamma \glipschitz \wdiameter\right)^{\nicefrac{3}{2}} \left(
    \ln\frac{1}{\delta}\right)^{\nicefrac{5}{4}}}{\epsilon^{\nicefrac{3}{2}}}
    \right) \eqperiod
  \end{align*}
  and with probability $1 - \delta$:
  \begin{equation*}
    f\left(\Pi_g\left( \bar{w} \right)\right) - f\left(w^*\right) \le
    \objective\left( \bar{w} \right) - \objective\left(w^*\right) \le \epsilon
    \;\;\;\;\;\; \mbox{and} \;\;\;\;\;\;
    \norm{\bar{w} - \Pi_g\left( \bar{w} \right)}_2 \le \frac{\epsilon}{\gamma
    \rho - \flipschitz}
    \eqcomma
  \end{equation*}
  where $w^* \in \{w \in \mathcal{W} : \forall i . g_i(w) \le 0\}$ is an
  arbitrary constraint-satisfying reference vector.
\end{thm}
\begin{prf}{light}
  Regardless of the value of $k$, it follows from
  \lemrefs{light-suboptimality}{full-suboptimality} that:
  \begin{equation*}
    \lightbound, \fullbound \le
    67 \sqrt{1 + \ln m} \wdiameter \left(\fgradbound + \gamma \ggradbound +
    \gamma \glipschitz \wdiameter\right) \sqrt{1 + \ln\frac{1}{\delta}}
    \sqrt{\frac{1}{T}}
    + \frac{8 \wdiameter \fgradbound \ln\frac{1}{\delta}}{3 T} \eqperiod
  \end{equation*}
  As in the proof of \thmref{full}, we define:
  \begin{align*}
    x =& \sqrt{T} \eqcomma \\
    c =& \frac{8}{3} \wdiameter \fgradbound \ln\frac{1}{\delta} \eqcomma \\
    b =& 67 \sqrt{1 + \ln m} \wdiameter \left(\fgradbound + \gamma \ggradbound
    + \gamma \glipschitz \wdiameter\right) \sqrt{1 + \ln\frac{1}{\delta}}
    \sqrt{\frac{1}{T}} \eqcomma \\
    a =& -\epsilon \eqcomma
  \end{align*}
  and consider the polynomial $0 = a x^2 + bx + c$. Any upper bound on all
  roots $x=\sqrt{T}$ of this polynomial will result in a lower-bound the values
  of $T$ for which $\lightbound, \fullbound \le \epsilon$ with probability
  $1-\delta$. By the Fujiwara bound~\citep{WikipediaPolynomialRoots}:
  \begin{equation*}
    T_{\epsilon} =
    \max\left\{ \frac{\left(134\right)^2 \left(1 + \ln m\right) \wdiameter^2
    \left( \fgradbound + \gamma \ggradbound + \gamma \glipschitz \wdiameter
    \right)^2 \left(1 + \ln\frac{1}{\delta}\right)}{\epsilon^2},
    \frac{16 \wdiameter \fgradbound \ln\frac{1}{\delta}}{3 \epsilon} \right\}
    \eqcomma
  \end{equation*}
  giving the claimed bound on $T_{\epsilon}$.
  For $C_{\epsilon}$, we observe that we will perform no more than $k+1$
  constraint checks at each iteration ($k+1$ by \lightalg if $k \le m$, and
  $m+1$ by \fullalg if $k > m$), and substitute the above bound on
  $T_{\epsilon}$ into the definition of $k$, yielding:
  \begin{align*}
    \left(k + 1\right) T_{\epsilon} \le &
    2T_{\epsilon} + m\left(1 + \ln m\right)^{\nicefrac{3}{4}} \sqrt{1 +
    \ln\frac{1}{\delta}} T_{\epsilon}^{\nicefrac{3}{4}} \sqrt{1 + \ln
    T_{\epsilon}} \\
    \le & \max\left\{ \frac{2 \left(134\right)^2 \left(1 + \ln m\right)
    \wdiameter^2 \left( \fgradbound + \gamma \ggradbound + \gamma \glipschitz
    \wdiameter \right)^2 \left(1 + \ln\frac{1}{\delta}\right)}{\epsilon^2},
    \frac{32 \wdiameter \fgradbound \ln\frac{1}{\delta}}{3 \epsilon} \right\}
    \\
    & + \max\left\{ \frac{\left(134\right)^{\nicefrac{3}{2}} m \left(1 + \ln
    m\right)^{\nicefrac{3}{2}} \wdiameter^{\nicefrac{3}{2}} \left( \fgradbound
    + \gamma \ggradbound + \gamma \glipschitz \wdiameter
    \right)^{\nicefrac{3}{2}} \left(1 +
    \ln\frac{1}{\delta}\right)^{\nicefrac{5}{4}}}{\epsilon^{\nicefrac{3}{2}}},
    \right. \\
    & \left. \left(\frac{16}{3}\right)^{\nicefrac{3}{4}} \frac{m \left(1 + \ln
    m\right)^{\nicefrac{3}{4}} \wdiameter^{\nicefrac{3}{4}}
    \fgradbound^{\nicefrac{3}{4}} \left(1 +
    \ln\frac{1}{\delta}\right)^{\nicefrac{5}{4}}}{\epsilon^{\nicefrac{3}{4}}}
    \right\} \sqrt{1 + \ln T_{\epsilon}}
    \eqperiod
  \end{align*}
  giving the claimed result (notice the $\sqrt{1 + \ln T_{\epsilon}}$ factor on
  the RHS, for which reason we have a $\tilde{O}$ bound on $C_{\epsilon}$,
  instead of $O$).
\end{prf}

%% file: figures/alg-mid.tex
\begin{algorithm*}[t]

\begin{pseudocode}
\codename \textbf{Hyperparameters:} $T_1$, $T_2$, $\eta$ \\
\codecommentline{First phase}\\
\codeline Initialize $w^{(1)} \in \mathcal{W}$ arbitrarily\\
\codeline For $t = 1$ to $T_1$:\\
\codeline \>Sample $\fgrad^{(t)}$ \codecomment{stochastic subgradient of $f(w^{(t)})$}\\
\codeline \>Let $\wgrad^{(t)} = \fgrad^{(t)} + \gamma \subgrad \max\{0, g(w^{(t)})\}$\\
\codeline \>Update $w^{(t+1)} = \Pi_w( w^{(t)} - (1 / \lambda t) \wgrad^{(t)} )$ \codecomment{$\Pi_w$ projects its argument onto $\mathcal{W}$ \wrt $\norm{\cdot}_2$} \\
\codecommentline{Second phase}\\
\codeline Average $w^{(T_1 + 1)} = (\sum_{t=1}^{T_1} w^{(t)}) / T_1$ \codecomment{initialize second phase to result of first}\\
\codeline Initialize $p^{(T_1 + 1)} \in \Delta^m$ to the uniform distribution\\
\codeline Initialize $\memgrad^{(T_1 + 1)}_j = \max\{0, g_j(w^{(T_1 + 1)})\}$\\
\codeline For $t = T_1 + 1$ to $T_1 + T_2$:\\
\codeline \>Sample $\fgrad^{(t)}$\\
\codeline \>Sample $i^{(t)} \sim p^{(t)}$\\
\codeline \>Let $\wgrad^{(t)} = \fgrad^{(t)} + \gamma \subgrad \max\{0, g_{i^{(t)}}(w^{(t)})\}$\\
\codeline \>Update $w^{(t+1)} = \Pi_w( w^{(t)} - (1 / \lambda t) \wgrad^{(t)} )$\\
\codeline \>Sample $j^{(t)} \sim \mbox{Unif}\{1,\dots,m\}$\\
\codeline \>Let $\pgrad^{(t)} = \gamma \memgrad^{(t)} + \gamma m e_{j^{(t)}} ( \max\{0, g_{j^{(t)}}(w^{(t)})\} -  \memgrad_{j^{(t)}}^{(t)} )$\\ 
\codeline \>Let $\memgrad^{(t+1)}_k = \memgrad^{(t)}_k$ if $k \ne j^{(t)}$, otherwise $\memgrad^{(t+1)}_{j^{(t)}} = \max\{0, g_{j^{(t)}}(w^{(t)})\}$\\
\codeline \>Update $\tilde{p}^{(t+1)} = \exp( \ln p^{(t)} + \eta \pgrad^{(t)} )$ \codecomment{element-wise $\exp$ and $\ln$}\\
\codeline \>Project $p^{(t+1)} = \tilde{p}^{(t+1)} / \norm{\tilde{p}^{(t+1)}}_1$\\
\codeline Average $\bar{w} = (\sum_{t=T_1 + 1}^{T_1 + T_2} w^{(t)}) / T_2$\\
\codeline Return $\Pi_g(\bar{w})$ \codecomment{optional if small constraint violations are acceptable}
\end{pseudocode}

\caption{
  \textbf{(\midname)} Minimizes a $\lambda$-strongly convex $f$ on
  $\mathcal{W}$ subject to the constraints $g_i(w) \le 0$ for
  $i\in\{1,\dots,m\}$. The algorithm consists of two phases: the first $T_1$
  iterations proceed like \fullalg, with every constraint being checked; the
  final $T_2$ iterations proceed like \lightalg, with only a constant number of
  constraints being checked during each iteration, and an auxiliary probability
  distribution $p$ being learned along the way.
  Notice that while second-phase $p$-update checks only one constraint, it,
  like \lightalg, requires $O(m)$ arithmetic operations. This issue is
  discussed further in \secref{practical:aggregation}.
}

\label{alg:mid}

\end{algorithm*}

%% file: theorems/thm-mid.tex
\begin{thm}{mid}
  \switchshowproofs{
    Suppose that the conditions of \lemrefs{projection}{mid-suboptimality}
    apply, with the $p$-update step size $\eta$ as defined in
    \lemref{mid-suboptimality}.
    If we run \algref{mid} (\midname) for $T_{\epsilon 1}$ iterations in the
    first phase and $T_{\epsilon 2}$ in the second:
  }{
    Suppose that the conditions of \lemref{projection} apply, with $g(w) =
    \max_i(g_i(w))$. Define $\fgradbound \ge \norm{\fgrad^{(t)}}_2$ and
    $\ggradbound \ge \norm{\subgrad \max(0, g_i(w))}_2$ as uniform upper bounds
    on the (stochastic) gradient magnitudes of $f$ and the $g_i$s,
    respectively, for all $i \in \{1,\dots,m\}$.
    We also assume that $f$ is $\lambda$-strongly convex, and that all $g_i$s
    are $\glipschitz$-Lipschitz \wrt $\norm{\cdot}_2$, \ie $\abs{g_i(w) -
    g_i(w')} \le \glipschitz \norm{w - w'}_2$ for all $w,w'\in\mathcal{W}$.
    If we run \algref{mid} (\midname) with the $p$-update step size $\eta =
    \lambda / 2 \gamma^2 \glipschitz^2$ for $T_{\epsilon 1}$ iterations in the
    first phase and $T_{\epsilon 2}$ in the second:
  }
  \begin{align*}
    T_{\epsilon 1} =&
    \tilde{O}\left(
    \frac{m \left(\ln m \right)^{\nicefrac{2}{3}} \left(\fgradbound + \gamma
    \ggradbound + \gamma \glipschitz \right)^{\nicefrac{4}{3}}}
    {\lambda^{\nicefrac{4}{3}} \epsilon^{\nicefrac{2}{3}}} +
    \frac{m^2 \left(\ln m\right) \left(\fgradbound + \gamma \ggradbound\right)}
    {\lambda \sqrt{\epsilon}}
    \right)
    \eqcomma \\
    T_{\epsilon 2} =&
    \tilde{O}\left(
    \frac{\left(\ln m \right) \left(\fgradbound + \gamma \ggradbound + \gamma
    \glipschitz \right)^2} {\lambda^2 \epsilon} +
    \frac{m^{\nicefrac{3}{2}} \left(\ln m\right)^{\nicefrac{3}{2}}
    \left(\fgradbound + \gamma \ggradbound\right)^{\nicefrac{3}{2}}}
    {\lambda^{\nicefrac{3}{2}} \epsilon^{\nicefrac{3}{4}}}
    \right)
    \eqcomma
  \end{align*}
  requiring a total of $C_{\epsilon}$ ``constraint checks'' (evaluations or
  differentiations of a single $g_i$):
  \begin{align*}
    C_{\epsilon} =&
    \tilde{O}\left(
    \frac{\left(\ln m \right) \left(\fgradbound + \gamma \ggradbound + \gamma
    \glipschitz \right)^2} {\lambda^2 \epsilon} +
    \frac{m^{\nicefrac{3}{2}} \left(\ln m\right)^{\nicefrac{3}{2}}
    \left(\fgradbound + \gamma \ggradbound\right)^{\nicefrac{3}{2}}}
    {\lambda^{\nicefrac{3}{2}} \epsilon^{\nicefrac{3}{4}}}
    \right. \\
    & \left. +
    \frac{m^2 \left(\ln m \right)^{\nicefrac{2}{3}} \left(\fgradbound + \gamma
    \ggradbound + \gamma \glipschitz \right)^{\nicefrac{4}{3}}}
    {\lambda^{\nicefrac{4}{3}} \epsilon^{\nicefrac{2}{3}}} +
    \frac{m^3 \left(\ln m\right) \left(\fgradbound + \gamma \ggradbound\right)}
    {\lambda \sqrt{\epsilon}}
    \right)
    \eqcomma
  \end{align*}
  then:
  \begin{equation*}
    \expectation\left[ \norm{\Pi_g(\bar{w}) - w^*}_2^2 \right] \le
    \expectation\left[ \norm{\bar{w} - w^*}_2^2 \right] \le \epsilon
    \eqcomma
  \end{equation*}
  where $w^* = \argmin_{\{w \in \mathcal{W} : \forall i . g_i(w) \le 0\}} f(w)$
  is the \emph{optimal} constraint-satisfying reference vector.
\end{thm}
\begin{prf}{mid}
  We begin by introducing a number $\tau \in \R$ with $\tau \ge 1$ that will be
  used to define the iteration counts $T_1$ and $T_2$ as:
  \begin{equation*}
    T_1 = \left\lceil m \tau^2 \right\rceil
    \;\;\;\;\;\; \mbox{and} \;\;\;\;\;\;
    T_2 = \left\lceil \tau^3 \right\rceil
    \eqperiod
  \end{equation*}
  By \lemref{mid-suboptimality}, the above definitions imply that:
  \begin{align*}
    \MoveEqLeft \expectation\left[ \norm{\Pi_g(\bar{w}) - w^*}_2^2 \right] \\
    & \le \frac{2 \left(\fgradbound + \gamma \ggradbound\right)^2 \left(4 + \ln
    m + 5 \ln \tau\right) + 8 \gamma^2 \glipschitz^2 \ln m}{\lambda^2 \tau^3} +
    \frac{3 m^4 \left(1 + \ln m\right)^2 \left(\fgradbound + \gamma
    \ggradbound\right)^2}{\lambda^2 m^2 \tau^4} \\
    & \le \frac{ 10 \left(1 + \ln m \right) \left(\fgradbound + \gamma
    \ggradbound + \gamma \glipschitz \right)^2 \left(1 + \ln \tau\right)
    }{\lambda^2 \tau^3} +
    \frac{3 m^2 \left(1 + \ln m\right)^2 \left(\fgradbound + \gamma
    \ggradbound\right)^2}{\lambda^2 \tau^4}
    \eqperiod
  \end{align*}
  Defining $\epsilon = \expectation\left[ \norm{\Pi_g(\bar{w}) - w^*}_2^2
  \right]$ and rearranging:
  \begin{align*}
    \MoveEqLeft \lambda^2 \epsilon \left(\frac{\tau}{\left(1 + \ln \tau
    \right)^{\nicefrac{1}{3}}}\right)^4 \\
    & \le 10 \left(1 + \ln m \right) \left(\fgradbound + \gamma \ggradbound +
    \gamma \glipschitz \right)^2 \left(\frac{\tau}{\left(1 + \ln \tau
    \right)^{\nicefrac{1}{3}}}\right) +
    3 m^2 \left(1 + \ln m\right)^2 \left(\fgradbound + \gamma
    \ggradbound\right)^2
    \eqperiod
  \end{align*}
  We will now upper-bound all roots of the above equation with a quantity
  $\tau_{\epsilon}$, for which all $\tau \ge \tau_{\epsilon}$ will result in
  $\epsilon$-suboptimality. By the Fujiwara
  bound~\citep{WikipediaPolynomialRoots}, and including the constraint that
  $\tau \ge 1$:
  \begin{align*}
    \frac{\tau_{\epsilon}}{\left(1 + \ln \tau_{\epsilon}
    \right)^{\nicefrac{1}{3}}} \le &
    \max\left\{
    1,
    2\left( \frac{10 \left(1 + \ln m \right) \left(\fgradbound + \gamma
    \ggradbound + \gamma \glipschitz \right)^2} {\lambda^2 \epsilon}
    \right)^{\nicefrac{1}{3}},
    \right. \\
    & \left.
    2\left( \frac{3 m^2 \left(1 + \ln m\right)^2 \left(\fgradbound + \gamma
    \ggradbound\right)^2} {2 \lambda^2 \epsilon} \right)^{\nicefrac{1}{4}}
    \right\}
    \eqperiod
  \end{align*}
  Substituting the above bound on $\tau_{\epsilon}$ into the definitions of
  $T_1$ and $T_2$ gives the claimed magnitudes of these $T_{\epsilon 1}$ and
  $T_{\epsilon 2}$, and using the fact that the $C_{\epsilon} = O(mT_{\epsilon
  1} + T_{\epsilon 2})$ gives the claimed bound on $C_{\epsilon}$.
\end{prf}

%% file: sec-comparison.tex
\section{Theoretical Comparison}\label{sec:comparison}

\input{figures/tab-comparison-light}
\input{figures/tab-comparison-mid}
\tabref{comparison-light} compares upper bounds on the convergence rates and
per-iteration costs when applied to a convex (but not necessarily strongly
convex) problem for \lightalg, \fullalg, projected SGD, the online Frank-Wolfe
algorithm of \citet{HazanKa12}, and a Frank-Wolfe-like online algorithm for
optimization over a polytope~\citep{GarberHa13}. The latter algorithm, which we
refer to as LLO-FW, achieves convergence rates comparable to projected SGD, but
uses a local linear oracle instead of a projection or full linear optimization.
To simplify the presentation, the dependencies on $\glipschitz$, $\fgradbound$
and $\ggradbound$ have been dropped---please refer to \thmrefs{light}{mid} and
the cited references for the complete statements.
\tabref{comparison-mid} contains the same comparison (without online
Frank-Wolfe) for $\lambda$-strongly convex problems.

At each iteration, all of these algorithms must find a stochastic subgradient
of $f$. In addition, each iteration of \lightalg and \midalg must perform
$O(m)$ arithmetic operations (for the $m$-dimensional vector operations used
when updating $p$)---this issue will be discussed further in
\secref{practical:aggregation}. However, projected SGD must project its iterate
onto the constraints \wrt the Euclidean norm, online Frank-Wolfe must perform a
linear optimization subject to the constraints, and LLO-FW must evaluate a
local linear oracle, which amounts to essentially local linear optimization.

\lightalg, \midalg and \fullalg share the same $\gamma$-dependence, but the
$m$-dependence of the convergence rate of \lightalg and \midalg is
logarithmically worse. The number of constraint evaluations, however, is
better: in the non-strongly convex case, ignoring all but the $m$ and
$\epsilon$ dependencies, \fullalg will check $O(m / \epsilon^2)$ constraints,
while \lightalg will check only $\tilde{O}((\ln m)/\epsilon^2 +
m/\epsilon^{\nicefrac{3}{2}})$, a significant improvement when $\epsilon$ is
small. Hence, particularly for problems with many expensive-to-evaluate
constraints, one would expect \lightalg to converge much more rapidly.
Likewise, for $\lambda$-strongly convex optimization, the dominant (in
$\epsilon$) terms in the bounds on the number of constraint evaluations go as
$m / \epsilon$ for \fullalg, and as $(\ln m)/\epsilon$ for \midalg, although
the lower-order terms in the \midalg bound are significantly more complex than
in the non-strongly convex case (see \tabref{comparison-mid} for full details).

Comparing with projected SGD, online Frank-Wolfe and LLO-FW is less
straightforward, not only because we're comparing upper bounds to upper bounds
(with all of the uncertainty that this entails), but also because we must
relate the value of $\gamma$ to the cost of performing the required projection,
constrained linear optimization or local linear oracle evaluation.
We note, however, that for non-strongly convex optimization, the
$\epsilon$-dependence of the convergence rate bound is worse for online
Frank-Wolfe ($1/\epsilon^3$) than for the other algorithms ($1/\epsilon^2$),
and that unless the constraints have some special structure, performing a
projection can be a very expensive operation.

For example, with general linear inequality constraints, each constraint check
performed by \lightalg, \midalg or \fullalg requires $O(d)$ time, whereas each
linear program optimized by online Frank-Wolfe could be solved in $O(d^2 m)$
time~\citep[Chapter 10.1]{Nemirovski04}, and each projection performed by SGD
in $O((dm)^{\nicefrac{3}{2}})$ time~\citep{GoldfarbLi91}. When the constraints
are taken to be arbitrary convex functions, instead of linear functions,
projections may be even more difficult.

We believe that in many cases $\gamma^2$ will be roughly on the order of the
dimension $d$, or number of constraints $m$, whichever is smaller, although it
can be worse for difficult constraint sets (see \secref{constraints:gamma}). In
practice, we have found that a surprisingly small $\gamma$---we use $\gamma=1$
in our experiments (\secref{experiments})---often suffices to result in
convergence to a feasible solution. With this in mind, and in light of the fact
that a fast projection, linear optimization, or local linear oracle evaluation
may only be possible for particular constraint sets, we believe that our
algorithms compare favorably with the alternatives.

%% file: figures/tab-comparison-light.tex
\begin{table*}[t]

\caption{
  Comparison of the number of iterations, and number of constraint checks,
  required to achieve $\epsilon$-suboptimality with high probability when
  optimizing a non-strongly-convex objective, up to constant and logarithmic
  factors, dropping the $\glipschitz$, $\fgradbound$ and $\ggradbound$
  dependencies, and ignoring the one-time cost of projecting the final result
  in \fullalg and \lightalg. For LLO-FW, the parameter to the local linear
  oracle has magnitude $O(\sqrt{d} \nu)$.
  See \secref{comparison}, \appref{full-light}, the non-smooth stochastic
  result of \citet[Theorem 4.3]{HazanKa12}, and \citet[Theorem 2]{GarberHa13}.
  Notice that because this table compares upper bounds to upper bounds,
  subsequent work may improve these bounds further.
  %
  %
}

\label{tab:comparison-light}

\begingroup
\renewcommand*{\arraystretch}{1.5}

\begin{center}

\begin{tabular}{r|c|c}
  \hline
  & \textbf{\#Iterations to achieve} & \textbf{\#Constraint checks to achieve} \\
  & \textbf{$\epsilon$-suboptimality} & \textbf{$\epsilon$-suboptimality} \\
  \hline
  \textbf{\fullalg} &
  $\frac{\gamma^2 \wdiameter^2}{\epsilon^2}$ &
  $\frac{m \gamma^2 \wdiameter^2}{\epsilon^2}$ \\
  \textbf{\lightalg} &
  $\frac{\left(\ln m\right) \gamma^2 \wdiameter^4}{\epsilon^2}$ &
  $\frac{\left(\ln m\right) \gamma^2 \wdiameter^4}{\epsilon^2} +
  \frac{m \left(\ln m\right)^{\nicefrac{3}{2}} \gamma^{\nicefrac{3}{2}} \wdiameter^3}{\epsilon^{\nicefrac{3}{2}}}$ \\
  \textbf{Projected SGD} &
  $\frac{\wdiameter^2}{\epsilon^2}$ &
  N/A (projection) \\
  \textbf{Online Frank-Wolfe} &
  $\frac{\wdiameter^3}{\epsilon^3}$ &
  N/A (linear optimization) \\
  \textbf{LLO-FW} &
  $\frac{d \nu^2 \wdiameter^2}{\epsilon^2}$ &
  N/A (local linear oracle) \\
  \hline
\end{tabular}

\end{center}

\endgroup

\end{table*}

%% file: figures/tab-comparison-mid.tex
\begin{table*}[t]

\caption{
  Same as \tabref{comparison-light}, except that the results bound the
  number of iterations or constraint checks required to achieve
  $\expectation[\norm{w-w^*}_2^2] \le \epsilon$, and the objective function is
  assumed to be $\lambda$-strongly convex. The bound given for \fullalg assumes
  that the constant $\eta$ used in \algref{full} has been replaced with the
  standard decreasing $1/\lambda t$ step size used in strongly-convex SGD. The
  \midalg bounds each contain four terms, listed in order of most-to-least
  dominant (in $\epsilon$). For LLO-FW, the parameter to the local linear
  oracle has magnitude $O(\sqrt{d} \nu)$.
  See \secref{comparison}, \appref{mid}, and \citet[Theorem 3]{GarberHa13}.
  %
  %
}

\label{tab:comparison-mid}

\begingroup
\renewcommand*{\arraystretch}{1.5}

\begin{center}

\begin{tabular}{r|c}
  \hline
  & \textbf{\#Iterations to achieve $\epsilon$-suboptimality} \\
  \hline
  \textbf{\fullalg} &
  $\frac{\gamma^2 \wdiameter^2}{\lambda^2 \epsilon}$ \\
  \textbf{\midalg} &
  $\frac{\left(\ln m\right) \gamma^2 \wdiameter^2}{\lambda^2 \epsilon}
  + \frac{m^{\nicefrac{3}{2}} \left(\ln m\right)^{\nicefrac{3}{2}} \gamma^{\nicefrac{3}{2}} \wdiameter^{\nicefrac{3}{2}}}{\lambda^{\nicefrac{3}{2}} \epsilon^{\nicefrac{3}{4}}}
  + \frac{m \left(\ln m\right)^{\nicefrac{2}{3}} \gamma^{\nicefrac{4}{3}} \wdiameter^{\nicefrac{4}{3}}}{\lambda^{\nicefrac{4}{3}} \epsilon^{\nicefrac{2}{3}}}
  + \frac{m^2 \left(\ln m\right) \gamma \wdiameter}{\lambda \sqrt{\epsilon}}$ \\
  \textbf{Projected SGD} &
  $\frac{\wdiameter^2}{\lambda^2 \epsilon}$ \\
  \textbf{LLO-FW} &
  $\frac{d \nu^2 \wdiameter^2}{\lambda^2 \epsilon}$ \\
  \hline
  & \textbf{\#Constraint checks to achieve $\epsilon$-suboptimality} \\
  \hline
  \textbf{\fullalg} &
  $\frac{m \gamma^2 \wdiameter^2}{\lambda^2 \epsilon}$ \\
  \textbf{\midalg} &
  $\frac{\left(\ln m\right) \gamma^2 \wdiameter^2}{\lambda^2 \epsilon}
  + \frac{m^{\nicefrac{3}{2}} \left(\ln m\right)^{\nicefrac{3}{2}} \gamma^{\nicefrac{3}{2}} \wdiameter^{\nicefrac{3}{2}}}{\lambda^{\nicefrac{3}{2}} \epsilon^{\nicefrac{3}{4}}}
  + \frac{m^2 \left(\ln m\right)^{\nicefrac{2}{3}} \gamma^{\nicefrac{4}{3}} \wdiameter^{\nicefrac{4}{3}}}{\lambda^{\nicefrac{4}{3}} \epsilon^{\nicefrac{2}{3}}}
  + \frac{m^3 \left(\ln m\right) \gamma \wdiameter}{\lambda \sqrt{\epsilon}}$ \\
  \textbf{Projected SGD} &
  N/A (projection) \\
  \textbf{LLO-FW} &
  N/A (local linear oracle) \\
  \hline
\end{tabular}

\end{center}

\endgroup

\end{table*}

%% file: sec-practical.tex
\section{Practical Considerations}\label{sec:practical}

\input{figures/alg-practical}
\algrefs{light}{mid} were designed primarily to be easy to analyze, but in real
world-applications we recommend making a few tweaks to improve performance.
The first of these is trivial: using a decreasing $w$-update step size
$\eta^{(t)}_w = \eta_w / \sqrt{t}$ when optimizing a non-strongly convex
objective, and $\eta^{(t)}_w = \eta_w / t$ for a strongly-convex objective. In
both cases we continue to use a constant $p$-update step size $\eta_p$. This
change, as well as that described in \secref{practical:minibatching}, is
included in \algref{practical}.

\subsection{Constraint Aggregation}\label{sec:practical:aggregation}

A natural concern about \algrefs{light}{mid} is that $O(m)$ arithmetic
operations are performed per iteration, even when only a few constraints are
checked.
When each constraint is expensive, this is a minor issue, since this cost will
be ``drowned out'' by that of checking the constraints. However, when the
constraints are very cheap, and the $O(m)$ arithmetic cost compares
disfavorably with the cost of checking a handful of constraints, it can become
a bottleneck.

Our solution to this issue is simple: transform a problem with a large number
of cheap constraints into one with a smaller number of more expensive
constraints. To this end, we partition the constraint indices $1,\dots,m$ into
$\tilde{m}$ sets $\{M_i\}$ of size at most $\lceil m/\tilde{m} \rceil$,
defining $\tilde{g}_i( w ) = \max_{j \in M_i} g_j( w )$, and then apply
\lightalg or \midalg on the $\tilde{m}$ aggregated constraints $\tilde{g}_i(w)
\le 0$.
This makes each constraint check $\lceil m / \tilde{m} \rceil$ times more
expensive, but reduces the dimension of $p$ from $m$ to $\tilde{m}$, shrinking
the per-iteration arithmetic cost to $O(\tilde{m})$.

\subsection{Automatic Minibatching}\label{sec:practical:minibatching}

Because \lightalg takes a minibatch size $k$ as a parameter, and the constants
from which we derive the recommended choice of $k$ (\thmref{light}) are often
unknown, a user is in the uncomfortable position of having to perform a
parameter search not only over the step sizes $\eta_w$ and $\eta_p$, but also
the minibatch size.
Furthermore, the fact that the theoretically-recommended $k$ is a decreasing
function of $T$ indicates that it might be better to check more constraints in
early iterations, and fewer in later ones.
Likewise, \midalg is structured as a two-phase algorithm, in which every
iteration checks every constraint in the first phase, and only a constant
number in the second, but it seems more sensible for the number of constraint
checks to decrease \emph{gradually} over time.

In addition, for both algorithms, it would be desirable to support separate
minibatching of the loss and constraint stochastic subgradients (\wrt $w$), in
which case there would be three minibatching parameters to determine: $k_f$,
$k_g$ and $k_p$. This makes things even harder for the user, since now there
are \emph{three} additional parameters that must be specified.

To remove the need to specify any minibatch-size hyperparameters, and to enable
the minibatch sizes to change from iteration-to-iteration, we propose a
heuristic that will automatically determine the minibatch sizes $k^{(t)}_f$,
$k^{(t)}_g$ and $k^{(t)}_p$ for each of the stochastic gradient components at
each iteration. Intuitively, we want to choose minibatch sizes in such a way
that the stochastic gradients are both cheap to compute and have low variance.
Our proposed heuristic does this by trading-off the computational cost and
``bound impact'' of the overall stochastic gradient, where the ``bound impact''
is a variance-like quantity that approximates the impact that taking a step
with particular minibatch sizes has on the relevant convergence rate bound.

Suppose that we're about to perform the $t$th iteration, and know that a
\emph{single} stochastic subgradient $\fgrad$ of $f(w)$ (corresponding to the
loss portion of $\wgrad$) has variance (more properly, covariance matrix trace)
$\bar{v}^{(t)}_f$ and requires a computational investment of $\bar{c}^{(t)}_f$
units.
Similarly, if we define $\ggrad$ by sampling $i \sim p$ and taking $\ggrad =
\gamma \subgrad \max\{0, g_i(w)\}$ (corresponding to the constraint portion of
$\wgrad$), then we can define variance and cost estimates of $\ggrad$ to be
$\bar{v}^{(t)}_g$ and $\bar{c}^{(t)}_g$, respectively. Likewise, we take
$\bar{v}^{(t)}_p$ and $\bar{c}^{(t)}_p$ to be estimates of the variance and
cost of a (non-minibatched version of) $\pgrad$.

In all three cases, the variance and cost estimates are those of a \emph{single
sample}, implying that a stochastic subgradient of $f(w)$ averaged over a
minibatch of size $k^{(t)}_f$ will have variance $\bar{v}^{(t)}_f / k^{(t)}_f$
and require a computational investment of $\bar{c}^{(t)}_f k^{(t)}_f$, and
likewise for the constraints and distribution.
In the context of \algref{practical}, with minibatch sizes of $k^{(t)}_f$,
$k^{(t)}_g$ and $k^{(t)}_p$, we define the overall bound impact $b$ and
computational cost $c$ of a single update as:
\begin{equation*}
  b = \frac{\eta^{(t)}_w \bar{v}^{(t)}_f}{k^{(t)}_f} + \frac{\eta^{(t)}_w
  \bar{v}^{(t)}_g}{k^{(t)}_g} + \frac{\eta_p \bar{v}^{(t)}_p}{k^{(t)}_p}
  \;\;\;\;\;\; \mbox{and} \;\;\;\;\;\;
  c = \bar{c}^{(t)}_f k^{(t)}_f + \bar{c}^{(t)}_g k^{(t)}_g + \bar{c}^{(t)}_p
  k^{(t)}_p
  \eqperiod
\end{equation*}
We should emphasize that the above definition of $b$ is merely a useful
approximation of how these quantities truly affect our bounds.

Given the three variance and three cost estimates, we choose minibatch sizes in
such a way as to minimize both the computational cost and bound impact of an
update. Imagine that we are given a fixed computational budget $c$. Then our
goal will be to choose the minibatch sizes in such a way that $b$ is minimized
for this budget, a problem that is easily solved in closed form:
\begin{equation*}
  \left[k^{(t)}_f, k^{(t)}_g, k^{(t)}_p\right] \propto
  \left[\sqrt{\frac{\eta^{(t)}_w \bar{v}^{(t)}_f}{\bar{c}^{(t)}_f}},
  \sqrt{\frac{\eta^{(t)}_w \bar{v}^{(t)}_g}{\bar{c}^{(t)}_g}},
  \sqrt{\frac{\eta_p \bar{v}^{(t)}_p}{\bar{c}^{(t)}_p}}\right]
  \eqperiod
\end{equation*}
We propose choosing the proportionality constant (and thereby the cost budget
$c$) in such a way that $k^{(t)}_f = 2$ (enabling us to calculate sample
variances, as explained below), and round the two other sizes to the nearest
integers, lower-bounding each so that $k^{(t)}_g \ge 2$ and $k^{(t)}_p \ge 1$.

While the variances and costs are not truly \emph{known} during optimization,
they are easy to estimate from known quantities. For the costs
$\bar{c}^{(t)}_f$, $\bar{c}^{(t)}_g$ and $\bar{c}^{(t)}_p$, we simply time how
long each past stochastic gradient calculation has taken, and then average them
to estimate the future costs. For the variances $\bar{v}^{(t)}_f$ and
$\bar{v}^{(t)}_g$, we restrict ourselves to minibatch sizes
$k^{(t)}_f,k^{(t)}_g \ge 2$, calculate the sample variances $v^{(t)}_f$ and
$v^{(t)}_g$ of the stochastic gradients at each iteration, and then average
over all past iterations (either uniformly, or a weighted average placing more
weight on recent iterations).

For $\bar{v}^{(t)}_p$, the situation is a bit more complicated, since the
$p$-updates are multiplicative (so we should use an $\ell^{\infty}$ variance)
and centered as in \eqref{johnson-zhang}.
%
%
Upper-bounding the $\ell^{\infty}$ norm with the $\ell^2$ norm and using the
fact that the minibatch $S^{(t)}$ is independently sampled yields the following
crude estimate:
%
%
\begin{equation*}
  v^{(t)}_p = \gamma^2 m^2 \left( \frac{1}{k^{(t)}_p} \sum_{i\in S^{(t)}}
  \left( \mu_i - \max\left\{0, g_i\left(w^{(t)}\right)\right\} \right)^2
  \right)
  \eqcomma
\end{equation*}
We again average $v^{(t)}_p$ across past iterations to estimate
$\bar{v}^{(t)}_p$.

%% file: figures/alg-practical.tex
\begin{algorithm*}[t]

\begin{pseudocode}
\codename \textbf{Hyperparameters:} $T$, $\eta_w$, $\eta_p$ \\
\codeline Initialize $w^{(1)} \in \mathcal{W}$ arbitrarily\\
\codeline Initialize $p^{(1)} \in \Delta^m$ to the uniform distribution\\
\codeline Initialize $\memgrad^{(1)}_j = \max\{0, g_j(w^{(1)})\}$ \codecomment{0 if $w^{(1)}$ is feasible}\\
\codeline For $t = 1$ to $T$:\\
\codeline \>Let $\eta^{(t)}_w = \eta_w / t$ if $f$ is strongly convex, $\eta_w / \sqrt{t}$ otherwise\\
\codeline \>Set $k^{(t)}_f$, $k^{(t)}_g$ and $k^{(t)}_p$ as described in \secref{practical:minibatching}\\
\codeline \>Sample $\fgrad^{(t)}_1, \dots, \fgrad^{(t)}_{k^{(t)}_f}$ \iid \codecomment{stochastic subgradients of $f(w^{(t)})$}\\
\codeline \>Sample $i^{(t)}_1, \dots, i^{(t)}_{k^{(t)}_g} \sim p^{(t)}$ \iid\\
\codeline \>Let $\wgrad^{(t)} = (1/k^{(t)}_f) \sum_{j=1}^{k^{(t)}_f} \fgrad^{(t)}_j + (\gamma / k^{(t)}_g) \sum_{j=1}^{k^{(t)}_g} \subgrad \max\{0, g_{i^{(t)}_j}(w^{(t)})\}$\\
\codeline \>Update $w^{(t+1)} = \Pi_w( w^{(t)} - \eta^{(t)}_w \wgrad^{(t)} )$ \codecomment{$\Pi_w$ projects its argument onto $\mathcal{W}$ \wrt $\norm{\cdot}_2$}\\
\codeline \>Sample $S^{(t)} \subseteq \{1,\dots,m\}$ with $\abs{S^{(t)}} = k^{(t)}_p$ uniformly without replacement\\
\codeline \>Let $\pgrad^{(t)} = \gamma \memgrad^{(t)} + (\gamma m / k^{(t)}_p) \sum_{j \in S^{(t)}} e_j ( \max\{0, g_j(w^{(t)})\} -  \memgrad_j^{(t)} )$\\ 
\codeline \>Let $\memgrad^{(t+1)}_j = \max\{0, g_j(w^{(t)})\}$ if $j \in S^{(t)}$, otherwise $\memgrad^{(t+1)}_j = \memgrad^{(t)}_j$\\
\codeline \>Update $\tilde{p}^{(t+1)} = \exp( \ln p^{(t)} + \eta_p \pgrad^{(t)} )$ \codecomment{element-wise $\exp$ and $\ln$}\\
\codeline \>Project $p^{(t+1)} = \tilde{p}^{(t+1)} / \norm{\tilde{p}^{(t+1)}}_1$\\
\codeline Average $\bar{w} = (\sum_{t=1}^T w^{(t)}) / T$\\
\codeline Return $\Pi_g(\bar{w})$ \codecomment{optional if small constraint violations are acceptable}
\end{pseudocode}

\caption{
  \textbf{(Practical \lightname)} Our proposed ``practical'' algorithm
  combining \lightalg and \midalg, along with the changes discussed in
  \secref{practical}.
}

\label{alg:practical}

\end{algorithm*}

%% file: sec-experiments.tex
\section{Experiments}\label{sec:experiments}

We validated the performance of our practical variant of \lightalg
(\algref{practical}) on a \youtube ranking problem in the style of
\citet{Joachims02}, in which the task is to predict what a user will watch
next, given that they have just viewed a certain video. In this setting, a user
has just viewed video $a$, was presented with a list of candidate videos to
watch next, and clicked on $b^+$, with $b^-$ being the video immediately
preceding $b^+$ in the list (if $b^+$ was the first list element, then the
example is thrown out).

We used an anonymized proprietary dataset consisting of $n = 612\,587$ training
pairs of feature vectors $(x^+,x^-)$, where $x^+$ is a vector of $12$ features
summarizing the similarity between $a$ and $b^+$, and $x^-$ between $a$ and
$b^-$.

We treat this as a standard pairwise ranking problem, for which the goal is to
estimate a function $f(\Phi(x)) = \inner{w}{\Phi(x)}$ such that $f(\Phi(x^+)) >
f(\Phi(x^-))$ for as many examples as possible, subject to the appropriate
regularization (or, in this case, constraints). Specifically, the
(unconstrained) learning task is to minimize the average empirical hinge loss:
\begin{equation*}
  \min_{w\in\mathcal{W}} \frac{1}{n} \sum_{i=1}^n \left( \max\left\{0, 1 -
  \inner{w}{\Phi\left(x_i^+\right) - \Phi\left(x_i^-\right)}\right\} \right)
  \eqperiod
\end{equation*}
All twelve of the features were designed to provide positive evidence---in
other words, if any one increases (holding the others fixed), then we expect
$f(\Phi(x))$ to increase. We have found that using constraints to enforce this
monotonicity property results in a better model in practice.

We define $\Phi(\cdot)$ as in lattice regression using simplex
interpolation~\citep{GarciaArGu12,GuptaCoPfVoCaMaMoEs16}, an approach which
works well at combining a small number of informative features, and more
importantly (for our purposes) enables one to force the learned function to be
monotonic via linear inequality constraints on the parameters. For the
resulting problem, the feature vectors have dimension $d=2^{12}=4096$, we chose
$\mathcal{W}$ to be defined by the box constraints $-10 \le w_i \le 10$ in each
of the $4096$ dimensions, and the total number of monotonicity-enforcing
linear inequality constraints is $m=24\,576$.

Every $\Phi(x)$ contains only $d+1=13$ nonzeros and can be computed in $O(d \ln
d)$ time. Hence, stochastic gradients of $f$ are inexpensive to compute.
Likewise, checking a monotonicity constraint only requires a single comparison
between two parameter values, so although there are a large number of them,
each constraint is very inexpensive to check.

\subsection{Implementations}

We implemented all algorithms in \cplusplus. Before running our main
experiments, we performed crude parameter searches on a power-of-four grid (\ie
$\ldots,\nicefrac{1}{16},\nicefrac{1}{4},1,4,16,\ldots$). For each candidate
value we performed roughly $10\,000$ iterations, and chose the parameter that
appeared to result in the fastest convergence in terms of the objective
function.

\paragraph \lightname Our implementation of \lightname includes all of the
suggested changes of \secref{practical}, including the constraint aggregation
approach of \secref{practical:aggregation}, although we used no aggregation
until our timing comparison (\secref{experiments:timing}).
For automatic minibatching, we took weighted averages of the variance estimates
as $\bar{v}^{(t+1)} \propto v^{(t)} + \nu \bar{v}^{(t)}$. We found that
up-weighting recent estimates (taking $\nu<1$) resulted in a noticeable
improvement, but that the precise value of $\nu$ mattered little (we used
$\nu=0.999$).
Based on the grid search described above, we chose $\gamma=1$, $\eta_w = 16$
and $\eta_p = \nicefrac{1}{16}$.

\paragraph \fullname Our \fullname implementation differs from that in
\algref{full} only in that we used a decreasing step size $\eta_w^{(t)} =
\eta_w / \sqrt{t}$. As with \lightname, we chose $\gamma = 1$ and $\eta_w = 16$
based on a grid search.

\paragraph {ProjectedSGD} We implemented Euclidean projections onto lattice
monotonicity constraints using
%
IPOPT \citep{WatcherBi06}
to optimize the resulting sparse $4096$-dimensional quadratic program. However,
the use of a QP solver for projected SGD---a very heavyweight
solution---resulted in an implementation that was too slow to experiment with,
requiring nearly four minutes per projection (observe that our experiments each
ran for millions of iterations).

\paragraph {ApproxSGD} This is an approximate projected SGD implementation
using the fast approximate update procedure described in
\citet{GuptaCoPfVoCaMaMoEs16}, which is an active set method that, starting
from the current iterate, moves along the boundary of the feasible region,
adding constraints to the active set as they are encountered, until the desired
step is exhausted (this is reminiscent of the local linear oracles considered
by \citet{GarberHa13}). This approach is particularly well-suited to this
particular constraint set because (1) when checking constraints for possible
inclusion in the active set, it exploits the sparsity of the stochastic
gradients to only consider monotonicity constraints which could possibly be
violated, and (2) projecting onto an intersection of active monotonicity
constraints reduces to uniformly averaging every set of parameters that are
``linked together'' by active constraints. Like the other algorithms, we used
step sizes of $\eta_w^{(t)} = \eta_w / \sqrt{t}$ and chose $\eta_w = 64$ based
on the grid search (recall that $\eta_w = 16$ was better for the other two
algorithms).

In every experiment we repeatedly looped over a random permutation of the
training set, and generated plots by averaging over $5$ such runs (with the
same $5$ random permutations) for each algorithm.

\subsection{Constraint-check Comparison}\label{sec:experiments:checks}

\input{figures/fig-checks}
\input{figures/fig-constraints}
In our first set of experiments, we compared the performance of \lightname,
\fullname and ApproxSGD in terms of the number of stochastic subgradients of
$f$ drawn, and the number of constraints checked. Because \lightname's
automatic minibatching fixes $k^{(t)}_f = 2$ (with the other two minibatch
sizes being automatically determined), in these experiments we used minibatch
sizes of $2$ for \fullalg and ApproxSGD, guaranteeing that all three algorithms
observe the same number of stochastic subgradients of $f$ at each iteration.

The left-hand plot of \figref{checks} shows that all three algorithms converge
at roughly comparable per-iteration rates, with ApproxSGD having a slight
advantage over \fullname, which itself converges a bit more rapidly than
\lightname.
The right-hand plot shows a striking difference, however---\lightname reaches a
near-optimal solution having checked more than $10\times$ fewer constraints
than \fullname.
Notice that we plot the suboptimalities of the projected iterates
$\Pi_w(w^{(t)})$ rather than of the $w^{(t)}$s themselves, in order to emulate
the final projection (line 7 of \algref{full} and 17 of \algref{practical}),
and guarantee that we only compare the average losses of \emph{feasible}
intermediate solutions.

In \figref{constraints}, we explore how well our algorithms enforce
feasibility, and how effective automatic minibatching is at choosing minibatch
sizes. The left-hand plot shows that both \fullname has converged to a
nearly-feasible solution after roughly $10\,000$ iterations, and \lightname
(unsurprisingly) takes more, perhaps $100\,000$ or so.
In the right-hand plot, we see that, in line with our expectations (see
\secref{practical:minibatching}), \lightname's automatic minibatching results
in very few constraints being checked in late iterations.

\subsection{Timing Comparison}\label{sec:experiments:timing}

\input{figures/fig-timing}
Our final experiment compared the wall-clock runtimes of our implementations.
Note that, because each monotonicity constraint can be checked with only a
single comparison (compare with \eg $O(d)$ arithmetic operations for a dense
linear inequality constraint), the $O(m)$ arithmetic cost of maintaining and
updating the probability distribution $p$ over the constraints is significant.
Hence, in terms of the constraint costs, this is nearly a worse-case problem
for \lightalg.
We experimented with power-of-$4$ constraint aggregate sizes
(\secref{practical:aggregation}), and found that using $\tilde{m}=96$
aggregated constraints, each of size $256$, worked best.

\fullalg, without minibatching, draws a single stochastic subgradient of $f$
and checks every constraint at each iteration. However, it would seem to be
more efficient to use minibatching to look at more stochastic subgradients at
each iteration, and therefore fewer constraints per stochastic subgradient of
$f$. Hence, for \fullalg, we again searched over power-of-$4$ minibatch sizes,
and found that $16$ worked best.

For ApproxSGD, the situation is less clear-cut. On the one hand, increasing the
minibatch size results in fewer approximate projections being performed per
stochastic subgradient of $f$. On the other, averaging more stochastic
subgradients results in less sparsity, slowing down the approximate projection.
We found that the latter consideration wins out---after searching again over
power-of-$4$ minibatch sizes, we found that a minibatch size of $1$ (\ie no
minibatching) worked best.

\figref{timing} contains the results of these experiments, showing that both
\fullalg and \lightalg converge significantly faster than ApproxSGD.
Interestingly, ApproxSGD is rather slow in early iterations (clipped off in
plot), but accelerates in later iterations. We speculate that the reason for
this behavior is that, close to convergence, the steps taken at each iteration
are smaller, and therefore the active sets constructed during the approximate
projection routine do not grow as large.
\fullalg enjoys a small advantage over \lightalg until both algorithms are very
close to convergence, but based on the results of \secref{experiments:checks},
we believe that this advantage would reverse if there were more constraints, or
if the constraints were more expensive to check.

%% file: figures/fig-checks.tex
\begin{figure*}[t]

\begin{center}

\begin{tabular}{lr}
  \includegraphics[width=\plotwidth]{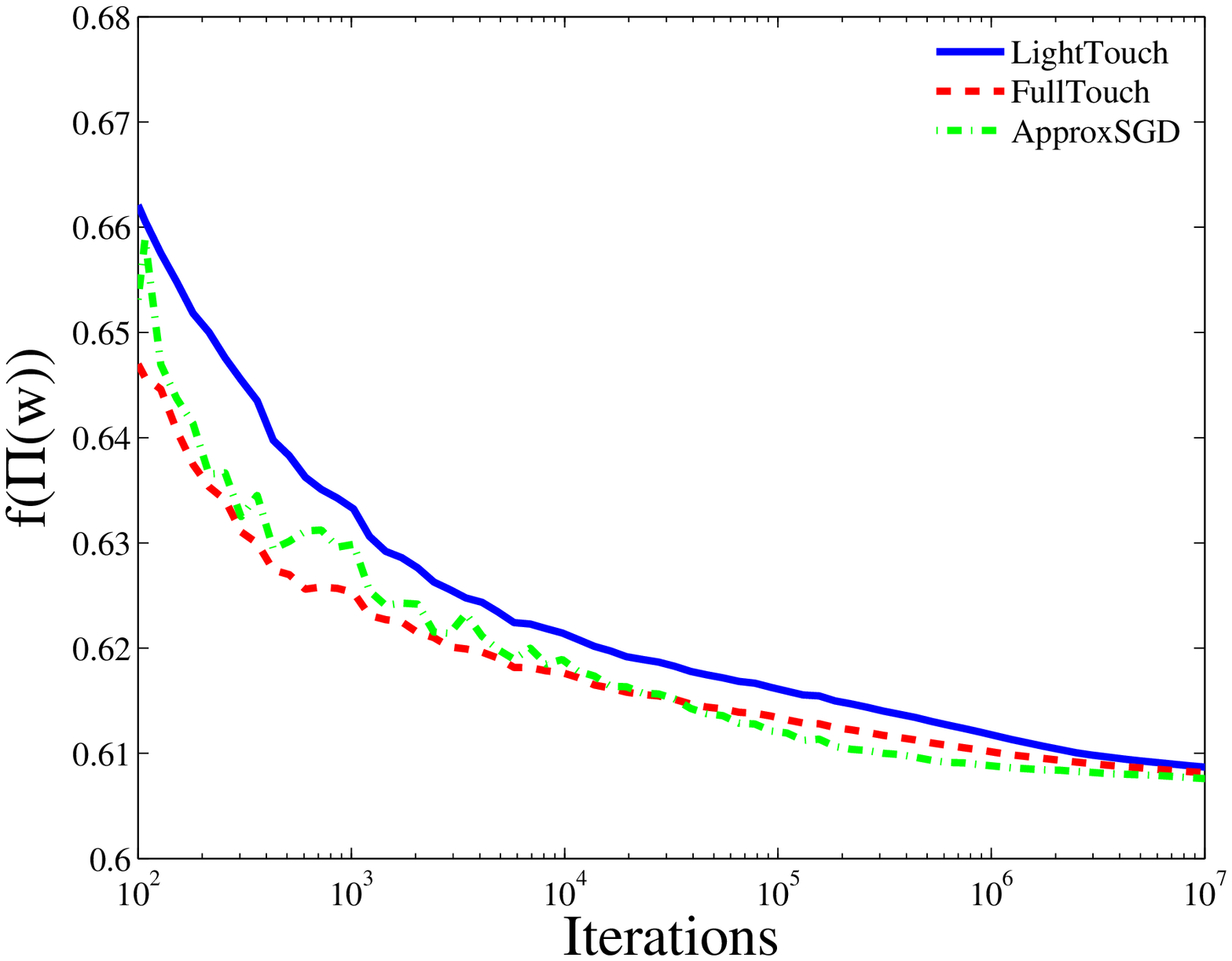} &
  \includegraphics[width=\plotwidth]{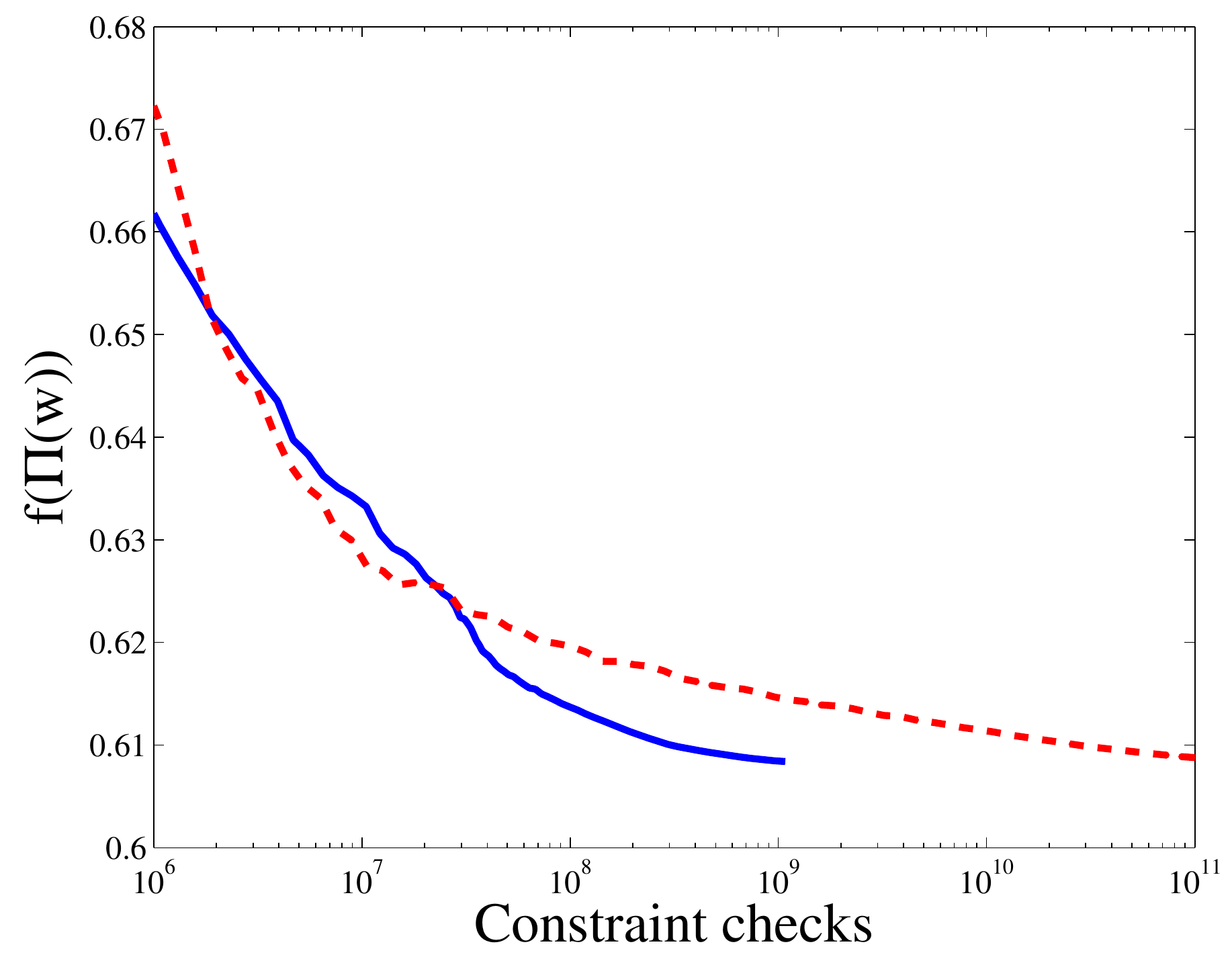}
\end{tabular}

\end{center}

\caption{
  Comparison of convergence rates of \lightalg, \fullalg and ApproxSGD on the
  \youtube ranking problem of \secref{experiments}. The two plots show the
  objective function (average training hinge loss) $f(\Pi_g(w^{(t)}))$ as a
  function of the number of iterations, and as a function of the total number
  of times a single constraint function $g_i$ was evaluated or differentiated,
  respectively.
}

\label{fig:checks}

\end{figure*}

%% file: figures/fig-constraints.tex
\begin{figure*}[t]

\begin{center}

\begin{tabular}{lr}
  \includegraphics[width=\plotwidth]{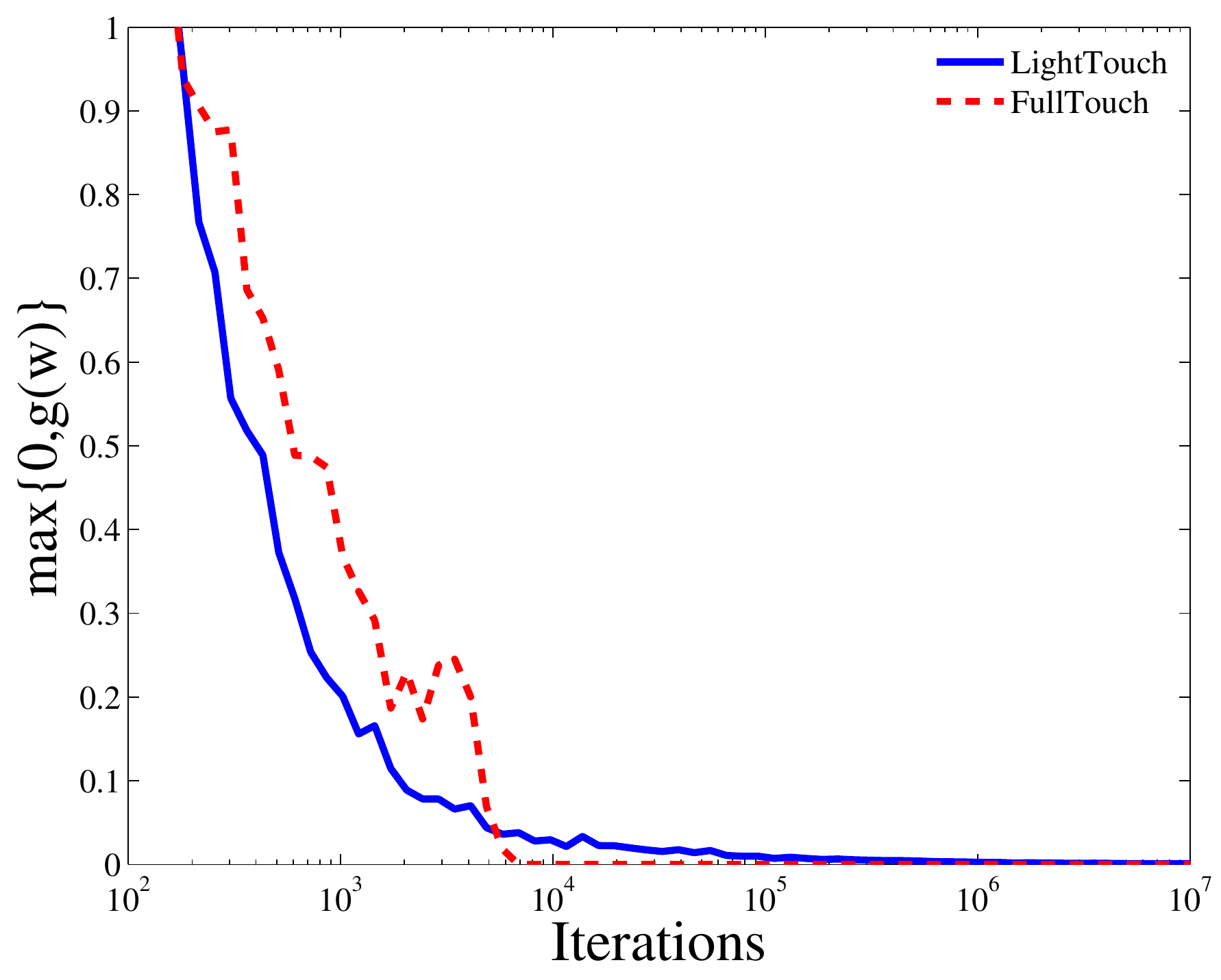} &
  \includegraphics[width=\plotwidth]{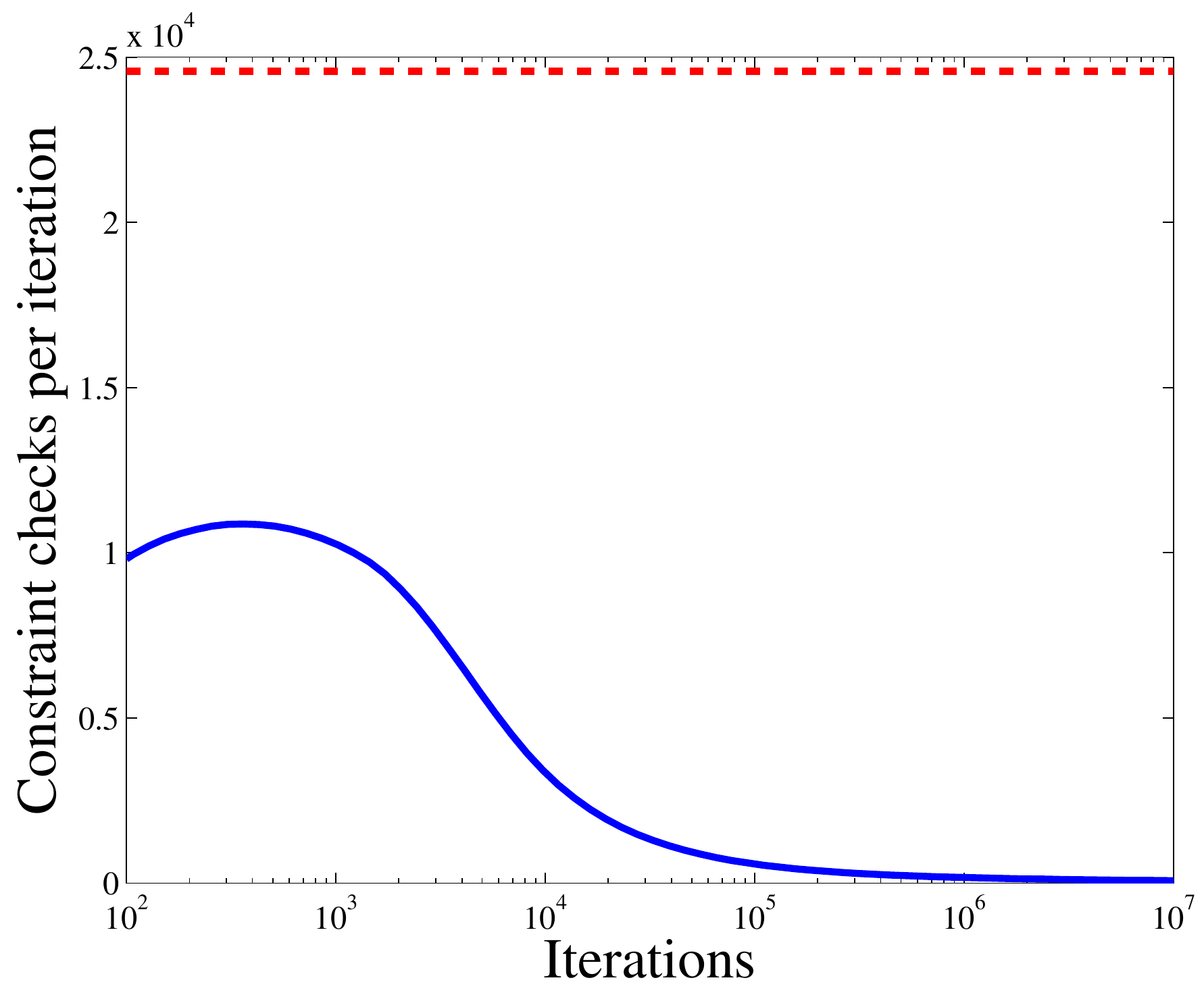}
\end{tabular}

\end{center}

\caption{
  Comparison constraint-handling of \lightalg and \fullalg on the \youtube
  ranking problem of \secref{experiments}. The two plots show the constraint
  violation magnitude $\max\{0, g(w^{(t)})\}$, and the average number of
  constraints checked per iteration up to this point, respectively, both as
  functions of the number of iterations.
}

\label{fig:constraints}

\end{figure*}

%% file: figures/fig-timing.tex
\begin{wrapfigure}{r}{0.5\textwidth}

\vspace{-4em}

\begin{center}

\includegraphics[width=\plotwidth]{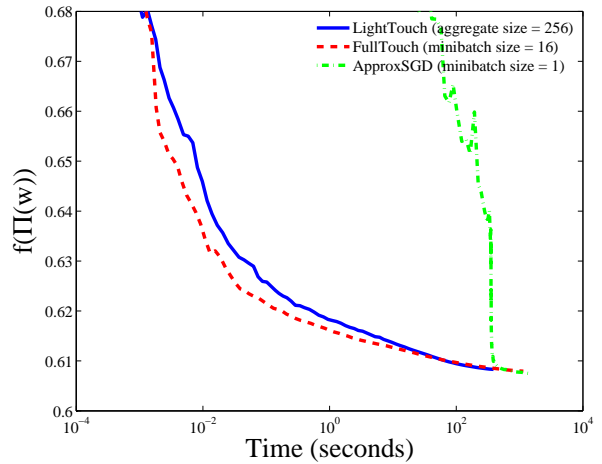}

\end{center}

\caption{
  Plot of the objective function (average training hinge loss)
  $f(\Pi_g(w^{(t)}))$ as a function of runtime for our implementations of
  \lightalg, \fullalg and ApproxSGD, on the \youtube ranking problem of
  \secref{experiments}.
}

\label{fig:timing}

\end{wrapfigure}

%% file: sec-conclusion.tex
\section{Conclusions}\label{sec:conclusion}

We have proposed an efficient strategy for large-scale heavily constrained
optimization, building on the work of \citet{MahdaviYaJiYi12}, and analyze its
performance, demonstrating that, asymptotically, our approach requires many
fewer constraint checks in order to converge.

We build on these theoretical results to propose a practical variant. The most
significant of these improvements is based on the observation that our
algorithm takes steps based on three separate stochastic gradients, and that
trading off the variances of computational costs of these three components is
beneficial. To this end, we propose a heuristic for dynamically choosing
minibatch sizes in such a way as to encourage faster convergence at a lower
computational cost.

Experiments on a real-world $4096$-dimensional machine learning problem with
$24\,576$ constraints and $612\,587$ training examples---too large for a
QP-based implementation of projected SGD---showed that our proposed method is
effective. In particular, we find that, in practice, our technique checks fewer
constraints per iteration than competing algorithms, and, as expected, checks
ever fewer as optimization progresses.

%% file: sec-acknowledgments.tex
\section*{Acknowledgments}

We thank Kevin Canini, Mahdi Milani Fard, Andrew Frigyik, Michael Friedlander
and Seungil You for helpful discussions, and proofreading earlier drafts.

%% file: app-mirror.tex
\section{Mirror Descent}\label{sec:mirror}

\input{figures/tab-notation-app-mirror}
Mirror descent~\citep{NemirovskiYu83,BeckTe03} is a meta-algorithm for
stochastic optimization (more generally, online regret minimization) which
performs gradient updates with respect to a meta-parameter, the \emph{distance
generating function} (\dgf). The two most widely-used {\dgf}s are the squared
Euclidean norm and negative Shannon entropy, for which the resulting MD
instantiations are stochastic gradient descent (SGD) and a multiplicative
updating algorithm, respectively. These are precisely the two {\dgf}s which our
constrained algorithm will use for the updates of $w$ and $p$.
We'll here give a number of results which differ only slightly from
``standard'' ones, beginning with a statement of an online MD bound adapted
from \citet{SrebroSrTe11}:

\medskip
\input{theorems/thm-mirror}
\medskip

It is straightforward to transform \thmref{mirror} into an in-expectation
result for stochastic subgradients:

\medskip
\input{theorems/cor-mirror-in-expectation}
\medskip

We next prove a high-probability analogue of the \corref{mirror-in-expectation},
based on a martingale bound of \citet{DzhaparidzeZa01}:

\medskip
\input{theorems/cor-mirror-high-probability}
\medskip

\algref{light} (\lightname) jointly optimizes over two sets of parameters, for
which the objective is convex in the first and linear (hence concave) in the
second. The convergence rate will be determined from a saddle-point bound,
which we derive from \corref{mirror-high-probability} by following
\citet{NemirovskiJuLaSh09,RakhlinSr13}, and simply applying it twice:

\medskip
\input{theorems/cor-mirror-saddle}

%% file: figures/tab-notation-app-mirror.tex
\begin{table*}[t]

\caption{New notation in \appref{mirror}.}

\label{tab:notation-app-mirror}

\begin{center}

\begin{tabular}{lll}
  \hline
  \textbf{Symbol} & \textbf{Description} & \textbf{Definition} \\
  \hline
  $\mathcal{A}$ & Convex domain of dual variables & \\
  $\filtration$ & Filtration & $\mdgrad^{(t)}$, $\wgrad^{(t)}$, $\alphagrad^{(t)}$ are $\filtration_t$-measurable \\
  $\mdnorm{\cdot}$, $\mddualnorm{\cdot}$ & Unspecified norm and its dual & \\
  $\wnorm{\cdot}$, $\wdualnorm{\cdot}$ & Norm on $\mathcal{W}$ and its dual & \\
  $\alphanorm{\cdot}$, $\alphadualnorm{\cdot}$ & Norm on $\mathcal{A}$ and its dual & \\
  $\Psi, \Psi^*$ & A \dgf and its convex conjugate & \\
  $\Psi_w, \Psi_w^*$ & A \dgf on $\mathcal{W}$ and its convex conjugate & \\
  $\Psi_{\alpha}, \Psi_{\alpha}^*$ & A \dgf on $\mathcal{A}$ and its convex conjugate & \\
  $\mdgrad$ & Stochastic subgradient & \\
  $\wgrad$ & Primal stochastic subgradient & \\
  $\alphagrad$ & Dual stochastic supergradient & \\
  $\mdnormradiusstar$ & Bound on $w^*$-centered radius of $\mathcal{W}$ & $\mdnormradiusstar \ge \mdnorm{w - w^*}$ \\
  $\wnormradiusstar$ & Bound on $w^*$-centered radius of $\mathcal{W}$ & $\wnormradiusstar \ge \wnorm{w - w^*}$ \\
  $\alphanormradiusstar$ & Bound on $\alpha^*$-centered radius of $\mathcal{A}$ & $\alphanormradiusstar \ge \alphanorm{\alpha - \alpha^*}$ \\
  $\sigma$ & Bound on $\mdgrad$ error & $\sigma \ge \mddualnorm{\expectation[\mdgrad^{(t)} \mid \filtration_{t-1}] - \mdgrad^{(t)}}$ \\
  $\sigma_w$ & Bound on $\wgrad$ error & $\sigma_w \ge \wdualnorm{\expectation[\wgrad^{(t)} \mid \filtration_{t-1}] - \wgrad^{(t)}}$ \\
  $\sigma_{\alpha}$ & Bound on $\alphagrad$ error & $\sigma_{\alpha} \ge \alphadualnorm{\expectation[\alphagrad^{(t)} \mid \filtration_{t-1}] - \alphagrad^{(t)}}$ \\
  $1 - \delta_{\sigma}$ & Probability that $\sigma$ bound holds & \\
  $1 - \delta_{\sigma w}$ & Probability that $\sigma_w$ bound holds & \\
  $1 - \delta_{\sigma \alpha}$ & Probability that $\sigma_{\alpha}$ bound holds & \\
  \hline
\end{tabular}

\end{center}

\end{table*}

%% file: theorems/thm-mirror.tex
\begin{thm}{mirror}
  Let $\mdnorm{\cdot}$ and $\mddualnorm{\cdot}$ be a norm and its dual. Suppose
  that the distance generating function (\dgf) $\Psi$ is $1$-strongly convex
  \wrt $\mdnorm{\cdot}$. Let $\Psi^*$ be the convex conjugate of $\Psi$, and
  take $\bregman_{\Psi}(w \vert w') = \Psi(w) - \Psi(w') - \inner{\grad
  \Psi(w')}{w - w'}$ to be the associated Bregman divergence.

  Take $f_t : \mathcal{W} \rightarrow \R$ to be a sequence of convex functions
  on which we perform $T$ iterations of mirror descent starting from $w^{(1)}
  \in \mathcal{W}$:
  \begin{align*}
    \tilde{w}^{(t+1)} &= \grad \Psi^* \left( \grad \Psi \left( w^{(t)} \right)
    - \eta \subgrad f_t \left(w^{(t)}\right) \right) \eqcomma \\
    w^{(t+1)} &= \underset{w \in \mathcal{W}}{\argmin} \bregman_{\Psi}\left(w
    \tallmid \tilde{w}^{(t+1)}\right) \eqcomma
  \end{align*}
  where $\subgrad f_t(w^{(t)}) \in \partial f_t(w^{(t)})$ is a subgradient of
  $f_t$ at $w^{(t)}$. Then:
  \iftoggle{istwocolumn}{
    \begin{align*}
      \MoveEqLeft \frac{1}{T} \sum_{t=1}^T \left( f_t\left(w^{(t)}\right) -
      f_t\left(w^*\right) \right) \\
      \le & \frac{\bregman_{\Psi}\left( w^* \tallmid
      w^{(1)} \right)}{\eta T} + \frac{\eta}{2 T} \sum_{t=1}^T
      \mddualnorm{\subgrad f_t \left(w^{(t)}\right)}^2 \eqcomma
    \end{align*}
  }{
    \begin{equation*}
      \frac{1}{T} \sum_{t=1}^T \left( f_t\left(w^{(t)}\right) -
      f_t\left(w^*\right) \right) \le \frac{\bregman_{\Psi}\left( w^* \tallmid
      w^{(1)} \right)}{\eta T} + \frac{\eta}{2 T} \sum_{t=1}^T
      \mddualnorm{\subgrad f_t \left(w^{(t)}\right)}^2 \eqcomma
    \end{equation*}
  }
  where  $w^* \in \mathcal{W}$ is an arbitrary reference vector.
\end{thm}
\begin{prf}{mirror}
  %
  %
  This proof is essentially the same as that of \citet[Lemma 2]{SrebroSrTe11}.
  By convexity:
  \begin{align*}
    \eta \left( f_t\left(w^{(t)}\right) - f_t\left(w^*\right) \right) & \le
    \inner{\eta \subgrad f_t \left(w^{(t)}\right)}{w^{(t)} - w^*} \\
    & \le \inner{\eta \subgrad f_t \left(w^{(t)}\right)}{w^{(t)} -
    \tilde{w}^{(t+1)}} + \inner{\eta \subgrad f_t
    \left(w^{(t)}\right)}{\tilde{w}^{(t+1)} - w^*} \eqperiod
  \end{align*}
  By H\"older's inequality, $\inner{w'}{w} \le \mdnorm{w'} \mddualnorm{w}$.
  Also, $\Psi(w) = \sup_{v} (\inner{v}{w} - \Psi^*(v))$ is maximized when
  $\grad \Psi^*(v) = w$, so $\grad \Psi( \grad \Psi^*(v) ) = v$. These results
  combined with the definition of $\tilde{w}^{(t+1)}$ give:
  \begin{align*}
    \eta \left( f_t\left(w^{(t)}\right) - f_t\left(w^*\right) \right) & \le
    \mddualnorm{\eta \subgrad f_t \left(w^{(t)}\right)} \mdnorm{w^{(t)} -
    \tilde{w}^{(t+1)}} \\
    & + \inner{ \grad \Psi \left( w^{(t)} \right) - \grad \Psi \left(
    \tilde{w}^{(t+1)} \right) }{\tilde{w}^{(t+1)} - w^*} \eqperiod
  \end{align*}
  Using Young's inequality and the definition of the Bregman divergence:
  \begin{align*}
    \eta \left( f_t\left(w^{(t)}\right) - f_t\left(w^*\right) \right) & \le
    \frac{1}{2}\mddualnorm{\eta \subgrad f_t \left(w^{(t)}\right)}^2 +
    \frac{1}{2}\mdnorm{w^{(t)} - \tilde{w}^{(t+1)}}^2 \\
    & + \bregman_{\Psi}\left( w^* \tallmid w^{(t)} \right) -
    \bregman_{\Psi}\left( w^* \tallmid \tilde{w}^{(t+1)} \right) -
    \bregman_{\Psi}\left( \tilde{w}^{(t+1)} \tallmid w^{(t)} \right) \eqperiod
  \end{align*}
  Applying the $1$-strong convexity of $\Psi$ to cancel the $\mdnorm{w^{(t)} -
  \tilde{w}^{(t+1)}}^2 / 2$ and $\bregman_{\Psi}( \tilde{w}^{(t+1)} \mid
  w^{(t)} ) $ terms:
  \begin{equation*}
    \eta \left( f_t\left(w^{(t)}\right) - f_t\left(w^*\right) \right) \le
    \frac{\eta^2}{2}\mddualnorm{\subgrad f_t \left(w^{(t)}\right)}^2 +
    \bregman_{\Psi}\left( w^* \tallmid w^{(t)} \right) - \bregman_{\Psi}\left(
    w^* \tallmid \tilde{w}^{(t+1)} \right) \eqperiod
  \end{equation*}
  Summing over $t$, using the nonnegativity of $\bregman_{\Psi}$, and dividing
  through by $\eta T$ gives the claimed result.
\end{prf}

%% file: theorems/cor-mirror-in-expectation.tex
\begin{cor}{mirror-in-expectation}
  %
  %
  Take $f_t : \mathcal{W} \rightarrow \R$ to be a sequence of convex
  functions, and $\filtration$ a filtration. Suppose that we perform $T$
  iterations of stochastic mirror descent starting from $w^{(1)} \in
  \mathcal{W}$, using the definitions of \thmref{mirror}:
  \begin{align*}
    \tilde{w}^{(t+1)} &= \grad \Psi^* \left( \grad \Psi \left( w^{(t)} \right)
    - \eta \mdgrad^{(t)} \right) \eqcomma \\
    w^{(t+1)} &= \underset{w \in \mathcal{W}}{\argmin} \bregman_{\Psi}\left(w
    \tallmid \tilde{w}^{(t+1)}\right) \eqcomma
  \end{align*}
  where $\mdgrad^{(t)}$ is a stochastic subgradient of $f_t$, \ie
  $\expectation[\mdgrad^{(t)} \mid \filtration_{t-1}] \in \partial
  f_t(w^{(t)})$, and $\mdgrad^{(t)}$ is $\filtration_t$-measurable.  Then:
  \begin{equation*}
    \frac{1}{T} \sum_{t=1}^T \expectation\left[ f_t\left(w^{(t)}\right) -
    f_t\left(w^*\right) \right]
    \le \frac{\bregman_{\Psi}\left( w^* \tallmid w^{(1)} \right)}{\eta T} +
    \frac{\eta}{2 T} \sum_{t=1}^T \expectation\left[
    \mddualnorm{\mdgrad^{(t)}}^2 \right] \eqcomma
  \end{equation*}
  where $w^* \in \mathcal{W}$ is an arbitrary reference vector.
\end{cor}
\begin{prf}{mirror-in-expectation}
  Define $\tilde{f}_t\left(w\right) = \inner{\mdgrad^{(t)}}{w}$, and observe
  that applying the non-stochastic MD algorithm of \thmref{mirror} to the
  sequence of functions $\tilde{f}_t$ results in the same sequence of iterates
  $w^{(t)}$ as does applying the above stochastic MD update to the sequence of
  functions $f_t$. Hence:
  \begin{equation}
    \label{eq:mirror-in-expectation:deterministic} \frac{1}{T} \sum_{t=1}^T
    \left( \tilde{f}_t\left(w^{(t)}\right) - \tilde{f}_t\left(w^*\right)
    \right) \le \frac{\bregman_{\Psi}\left( w^* \tallmid w^{(1)} \right)}{\eta
    T} + \frac{\eta}{2 T} \sum_{t=1}^T \mddualnorm{\mdgrad^{(t)}}^2 \eqperiod
  \end{equation}
  By convexity, $f_t(w^{(t)}) - f_t(w^*) \le \inner{\expectation[\mdgrad^{(t)}
  \mid \filtration_{t-1}]}{w^{(t)} - w^*}$, while $\tilde{f}_t(w^{(t)}) -
  \tilde{f}_t(w^*) = \inner{\mdgrad^{(t)}}{w^{(t)} - w^*} $ by definition.
  Taking expectations of both sides of
  \eqref{mirror-in-expectation:deterministic} and plugging in these
  inequalities yields the claimed result.
\end{prf}

%% file: theorems/cor-mirror-high-probability.tex
\begin{cor}{mirror-high-probability}
  %
  %
  In addition to the assumptions of \corref{mirror-in-expectation}, suppose that,
  with probability $1 - \delta_{\sigma}$, $\sigma$ satisfies the
  following uniformly for all $t \in \{1,\dots,T\}$:
  \begin{equation*}
    \mddualnorm{\expectation\left[\mdgrad^{(t)} \tallmid
    \filtration_{t-1}\right] - \mdgrad^{(t)}} \le \sigma \eqperiod
  \end{equation*}
  Then, with probability $1 - \delta_{\sigma} - \delta$, the
  above $\sigma$ bound will hold, and:
  \begin{equation*}
    \frac{1}{T} \sum_{t=1}^T \left( f_t\left(w^{(t)}\right) -
    f_t\left(w^*\right) \right)
    \le \frac{\bregman_{\Psi}\left( w^* \tallmid w^{(1)} \right)}{\eta T} +
    \frac{\eta}{2 T} \sum_{t=1}^T \mddualnorm{\mdgrad^{(t)}}^2
    + \frac{\sqrt{2} \mdnormradiusstar \sigma
    \sqrt{\ln\frac{1}{\delta}}}{\sqrt{T}} + \frac{2 \mdnormradiusstar \sigma
    \ln\frac{1}{\delta}}{3 T} \eqcomma
  \end{equation*}
  where $w^* \in \mathcal{W}$ is an arbitrary reference vector and
  $\mdnormradiusstar \ge \sup_{w \in \mathcal{W}} \mdnorm{w - w^*}$ bounds the
  radius of $\mathcal{W}$ centered on $w^*$.
\end{cor}
\begin{prf}{mirror-high-probability}
  Define $\tilde{f}_t\left(w\right) = \inner{\mdgrad^{(t)}}{w}$ as in the proof
  of \corref{mirror-in-expectation}, and observe that
  \eqref{mirror-in-expectation:deterministic} continues to apply.
  Define a sequence of random variables $M_0 = 0$, $M_{t} = M_{t-1} +
  \inner{\expectation[\mdgrad^{(t)} \mid \filtration_{t-1}] -
  \mdgrad^{(t)}}{w^{(t)} - w^*}$, and notice that $M$ forms a martingale \wrt
  the filtration $\filtration$. From this definition, H\"older's inequality
  gives that:
  \begin{equation*}
    \abs{M_{t} - M_{t-1}}
    %
    %
    \le \mddualnorm{\expectation\left[\mdgrad^{(t)} \tallmid
    \filtration_{t-1}\right] - \mdgrad^{(t)}} \mdnorm{w^{(t)} - w^*}
    \le \mdnormradiusstar \sigma \eqperiod
  \end{equation*}
  the above holding with probability $1-\delta_{\sigma}$. Plugging $a =
  \mdnormradiusstar \sigma$ and $L = T \mdnormradiusstar^2 \sigma^2$ into the
  Bernstein-type martingale inequality of \citet[Theorem 3.3]{DzhaparidzeZa01}
  gives:
  \begin{equation*}
    \probability\left\{ \frac{1}{T} M_T \ge \epsilon \right\} \le
    \delta_{\sigma} + \exp\left( -\frac{3 T \epsilon^2}{6 \mdnormradiusstar^2
    \sigma^2 + 2 \mdnormradiusstar \sigma \epsilon} \right) \eqperiod
  \end{equation*}
  Solving for $\epsilon$ using the quadratic formula and upper-bounding gives
  that, with probability $1 - \delta_{\sigma} - \delta$:
  \begin{equation*}
    \frac{1}{T} \sum_{t=1}^T \inner{\expectation\left[\mdgrad^{(t)} \tallmid
    \filtration_{t-1}\right] - \mdgrad^{(t)}}{w^{(t)} - w^*} \le \frac{\sqrt{2}
    \mdnormradiusstar \sigma \sqrt{\ln\frac{1}{\delta}}}{\sqrt{T}} + \frac{2
    \mdnormradiusstar \sigma \ln\frac{1}{\delta}}{3 T} \eqperiod
  \end{equation*}
  As in the proof of \corref{mirror-in-expectation}, $f_t(w^{(t)}) - f_t(w^*)
  \le \inner{\expectation[\mdgrad^{(t)} \mid \filtration_{t-1}]}{w^{(t)} -
  w^*}$, while $\tilde{f}_t(w^{(t)}) - \tilde{f}_t(w^*) =
  \inner{\mdgrad^{(t)}}{w^{(t)} - w^*} $ by definition, which combined with
  \eqref{mirror-in-expectation:deterministic} yields the claimed result.
\end{prf}

%% file: theorems/cor-mirror-saddle.tex
\begin{cor}{mirror-saddle}
  Let $\wnorm{\cdot}$ and $\alphanorm{\cdot}$ be norms with duals
  $\wdualnorm{\cdot}$ and $\alphadualnorm{\cdot}$. Suppose that $\Psi_w$ and
  $\Psi_{\alpha}$ are $1$-strongly convex \wrt $\wnorm{\cdot}$ and
  $\alphanorm{\cdot}$, have convex conjugates $\Psi_w^*$ and $\Psi_{\alpha}^*$,
  and associated Bregman divergences $\bregman_{\Psi_w}$ and
  $\bregman_{\Psi_{\alpha}}$, respectively.

  Take $f:\mathcal{W} \times \mathcal{A} \rightarrow \R$ to be convex in its
  first parameter and concave in its second, let $\filtration$ be a filtration,
  and suppose that we perform $T$ iterations of MD:
  \iftoggle{istwocolumn}{
    \begin{align*}
      \tilde{w}^{(t+1)} &= \grad \Psi_w^* \left( \grad \Psi_w \left( w^{(t)}
      \right) - \eta \wgrad^{(t)} \right) \eqcomma \\
      w^{(t+1)} &= \underset{w \in \mathcal{W}}{\argmin}
      \bregman_{\Psi_w}\left(w \tallmid \tilde{w}^{(t+1)}\right) \eqcomma \\
      \tilde{\alpha}^{(t+1)} &= \grad \Psi_{\alpha}^* \left( \grad
      \Psi_{\alpha} \left( \alpha^{(t)} \right) + \eta \alphagrad^{(t)}
      \right) \eqcomma \\
      \alpha^{(t+1)} &= \underset{\alpha \in \mathcal{A}}{\argmin}
      \bregman_{\Psi_{\alpha}}\left(\alpha \tallmid
      \tilde{\alpha}^{(t+1)}\right) \eqcomma
    \end{align*}
  }{
    \begin{align*}
      \begin{aligned}
        \tilde{w}^{(t+1)} &= \grad \Psi_w^* \left( \grad \Psi_w \left( w^{(t)}
        \right) - \eta \wgrad^{(t)} \right) \eqcomma \\
        w^{(t+1)} &= \underset{w \in \mathcal{W}}{\argmin}
        \bregman_{\Psi_w}\left(w \tallmid \tilde{w}^{(t+1)}\right) \eqcomma
      \end{aligned}
      & \hspace{4em}
      \begin{aligned}
        \tilde{\alpha}^{(t+1)} &= \grad \Psi_{\alpha}^* \left( \grad
        \Psi_{\alpha} \left( \alpha^{(t)} \right) + \eta \alphagrad^{(t)}
        \right) \eqcomma \\
        \alpha^{(t+1)} &= \underset{\alpha \in \mathcal{A}}{\argmin}
        \bregman_{\Psi_{\alpha}}\left(\alpha \tallmid
        \tilde{\alpha}^{(t+1)}\right) \eqcomma
      \end{aligned}
    \end{align*}
  }
  where $\wgrad^{(t)}$ is a stochastic subgradient of $f(w^{(t)},
  \alpha^{(t)})$ \wrt its first parameter, and $\alphagrad^{(t)}$ a stochastic
  supergradient \wrt its second, with both $\wgrad^{(t)}$ and
  $\alphagrad^{(t)}$ being $\filtration_t$-measurable.
  We assume that, with probabilities  $1 - \delta_{\sigma w}$ and $1 -
  \delta_{\sigma \alpha}$ (respectively), $\sigma_w^2$ and $\sigma_{\alpha}^2$
  satisfy the following uniformly for all $t \in \{1,\dots,T\}$:
  \begin{equation*}
    \wdualnorm{\expectation\left[\wgrad^{(t)} \tallmid \filtration_{t-1}\right]
    - \wgrad^{(t)}} \le \sigma_w
    \;\;\;\;\;\; \mbox{and} \;\;\;\;\;\;
    \alphadualnorm{\expectation\left[\alphagrad^{(t)} \tallmid
    \filtration_{t-1}\right] - \wgrad^{(t)}}
    \le \sigma_{\alpha} \eqperiod
  \end{equation*}
  Under these conditions, with probability $1 - \delta_{\sigma w} -
  \delta_{\sigma \alpha} - 2 \delta$, the above $\sigma_w$ and
  $\sigma_{\alpha}$ bounds will hold, and:
  \iftoggle{istwocolumn}{
    \begin{align*}
      \MoveEqLeft \frac{1}{T} \sum_{t=1}^T \left( f\left(w^{(t)}, \alpha^*\right) -
      f\left(w^*, \alpha^{(t)}\right) \right) \\
      \le & \frac{\bregman_{\Psi_w}\left(w^* \tallmid w^{(1)}\right) +
      \bregman_{\Psi_{\alpha}}\left(\alpha^* \tallmid \alpha^{(1)}\right)}{\eta
      T} \\
      & + \frac{\eta}{2 T} \sum_{t=1}^T \left( \wdualnorm{\wgrad^{(t)}}^2 +
      \alphadualnorm{\alphagrad^{(t)}}^2 \right) \\
      & + \frac{\sqrt{2} \left( \wnormradiusstar \sigma_w +
      \alphanormradiusstar \sigma_{\alpha} \right)
      \sqrt{\ln\frac{1}{\delta}}}{\sqrt{T}} \\
      & + \frac{2 \left( \wnormradiusstar \sigma_w + \alphanormradiusstar
      \sigma_{\alpha} \right) \ln\frac{1}{\delta}}{3 T} \eqcomma
    \end{align*}
  }{
    \begin{align*}
      \MoveEqLeft \frac{1}{T} \sum_{t=1}^T \left( f\left(w^{(t)}, \alpha^*\right) -
      f\left(w^*, \alpha^{(t)}\right) \right) \\
      \le & \frac{\bregman_{\Psi_w}\left(w^* \tallmid w^{(1)}\right) +
      \bregman_{\Psi_{\alpha}}\left(\alpha^* \tallmid \alpha^{(1)}\right)}{\eta
      T}
      + \frac{\eta}{2 T} \sum_{t=1}^T \left( \wdualnorm{\wgrad^{(t)}}^2 +
      \alphadualnorm{\alphagrad^{(t)}}^2 \right) \\
      & + \frac{\sqrt{2} \left( \wnormradiusstar \sigma_w +
      \alphanormradiusstar \sigma_{\alpha} \right)
      \sqrt{\ln\frac{1}{\delta}}}{\sqrt{T}}
      + \frac{2 \left( \wnormradiusstar \sigma_w + \alphanormradiusstar
      \sigma_{\alpha} \right) \ln\frac{1}{\delta}}{3 T} \eqcomma
    \end{align*}
  }
  where $w^* \in \mathcal{W}$ and $\alpha^* \in \mathcal{A}$ are arbitrary
  reference vectors, and
  $\wnormradiusstar \ge \wnorm{w - w^*}$ and $\alphanormradiusstar \ge
  \alphanorm{\alpha - \alpha^*}$ bound the radii of $\mathcal{W}$ and
  $\mathcal{A}$ centered on $w^*$ and $\alpha^*$, respectively.
\end{cor}
\begin{prf}{mirror-saddle}
  This is a convex-concave saddle-point problem, which we will optimize by
  playing two convex optimization algorithms against each other, as in
  \citet{NemirovskiJuLaSh09,RakhlinSr13}.
  By \corref{mirror-high-probability}, with probability $1 - \delta_{\sigma w} -
  \delta$ and $1 - \delta_{\sigma \alpha} - \delta$, respectively:
  \begin{align*}
    \MoveEqLeft \frac{1}{T} \sum_{t=1}^T \left( f\left(w^{(t)}, \alpha^{(t)}\right) -
    f\left(w^*, \alpha^{(t)}\right) \right) \\
    \le & \frac{\bregman_{\Psi_w}\left(w^* \tallmid w^{(1)}\right)}{\eta T} +
    \frac{\eta}{2 T}\sum_{t=1}^T \wdualnorm{\wgrad^{(t)}}^2
    + \frac{\sqrt{2} \wnormradiusstar \sigma_w
    \sqrt{\ln\frac{1}{\delta}}}{\sqrt{T}} + \frac{2 \wnormradiusstar \sigma_w
    \ln\frac{1}{\delta}}{3 T} \eqcomma \\
    \MoveEqLeft \frac{1}{T} \sum_{t=1}^T \left( f\left(w^{(t)}, \alpha^*\right) -
    f\left(w^{(t)}, \alpha^{(t)}\right) \right) \\
    \le & \frac{\bregman_{\Psi_{\alpha}}\left(\alpha^* \tallmid
    \alpha^{(1)}\right)}{\eta T} + \frac{\eta}{2 T}\sum_{t=1}^T
    \alphadualnorm{\alphagrad^{(t)}}^2
    + \frac{\sqrt{2} \alphanormradiusstar \sigma_{\alpha}
    \sqrt{\ln\frac{1}{\delta}}}{\sqrt{T}} + \frac{2 \alphanormradiusstar
    \sigma_{\alpha} \ln\frac{1}{\delta}}{3 T} \eqperiod
  \end{align*}
  Adding these two inequalities gives the claimed result.
\end{prf}

%% file: app-sgd.tex
\section{SGD for Strongly-Convex Functions}\label{sec:sgd}

For $\lambda$-strongly convex objective functions, we can achieve a faster
convergence rate for SGD by using the step sizes $\eta_t = 1/\lambda t$. Our
eventual algorithm (\algref{mid}) for strongly-convex heavily-constrained
optimization will proceed in two phases, with the second phase ``picking up''
where the first phase ``left off'', for which reason we present a convergence
rate, based on \citet[Lemma 2]{ShalevSiSrCo10}, that effectively starts at
iteration $T_0$ by using the step sizes $\eta_t = 1/\lambda (T_0 + t)$:

\medskip
\input{theorems/thm-sgd}
\medskip

As we did \appref{mirror}, we convert this into a result for stochastic
subgradients:

\medskip
\input{theorems/cor-sgd-in-expectation}
\medskip

We now use this result to prove an in-expectation saddle point bound:

\medskip
\input{theorems/cor-sgd-saddle}

%% file: theorems/thm-sgd.tex
\begin{thm}{sgd}
  Take $f_t:\mathcal{W}\rightarrow\R$ to be a sequence of $\lambda$-strongly
  convex functions on which we perform $T$ iterations of stochastic gradient
  descent starting from $w^{(1)} \in \mathcal{W}$:
  \begin{equation*}
    w^{(t + 1)} = \Pi_w\left(w^{(t)} - \eta_t\subgrad
    f_t\left(w^{(t)}\right)\right) \eqcomma
  \end{equation*}
  where $\subgrad f_t\left(w^{(t)}\right)\in\partial f_t\left(w^{(t)}\right)$
  is a subgradient of $f_t$ at $w^{(t)}$, and $\norm{\subgrad
  f_t\left(w^{(t)}\right)}_2 \le \mdgradbound$ for all $t$.
  If we choose $\eta_t=\frac{1}{\lambda\left(T_0 + t\right)}$ for some
  $T_0\in\N$, then:
  \begin{equation*}
    \frac{1}{T}\sum_{t=1}^T \left( f_t\left( w^{(t)} \right) - f_t\left(w^{*}\right) \right) \le
    \frac{\mdgradbound^2 \left(1 + \ln
    T\right)}{2\lambda T} + \frac{\lambda T_0}{2T} \norm{w^{(1)} - w^{*}}_2^2 \eqcomma
  \end{equation*}
  where $w^* \in \mathcal{W}$ is an arbitrary reference vector and $\mdgradbound \ge
  \norm{\subgrad f_t\left(w^{(t)}\right)}_2$ bounds the subgradient norms for
  all $t$.
\end{thm}
\begin{prf}{sgd}
  This is nothing but a small tweak to \citet[Lemma 2]{ShalevSiSrCo10}.
  Starting from Equations 10 and 11 of that proof:
  \begin{align*}
    \MoveEqLeft \sum_{t=1}^T \left( f_t\left(w^{(t)}\right) -
    f_t\left(w^{*}\right) \right) \\
    \le & \frac{\mdgradbound^2}{2}\sum_{t=1}^T\eta_t + \sum_{t=1}^T\left(\frac{1}{2\eta_t}
    - \frac{\lambda}{2}\right)\norm{w^{(t)} - w^{*}}_2^2 -
    \sum_{t=1}^T\frac{1}{2\eta_t}\norm{w^{(t + 1)} - w^{*}}_2^2 \eqperiod
  \end{align*}
  Taking $\eta_t=\frac{1}{\lambda\left(T_0 + t\right)}$:
  \begin{align*}
    \MoveEqLeft \sum_{t=1}^T \left( f_t\left(w^{(t)}\right) -
    f_t\left(w^{*}\right) \right) \\
    \le & \frac{\mdgradbound^2}{2\lambda}\left(\frac{1}{T_0 + 1} + \int_{t=T_0 + 1}^{T_0 +
    T}\frac{dt}{t}\right) + \frac{\lambda T_0}{2}\norm{w^{(1)} - w^{*}}_2^2 -
    \frac{\lambda\left(T_0 + T\right)}{2}\norm{w^{(T + 1)} - w^{*}}_2^2 \eqperiod
  \end{align*}
  Dividing through by $T$, simplifying and bounding yields the claimed result.
\end{prf}

%% file: theorems/cor-sgd-in-expectation.tex
\begin{cor}{sgd-in-expectation}
  Take $f_t:\mathcal{W}\rightarrow\R$ to be a sequence of $\lambda$-strongly
  convex functions, and $\filtration$ a filtration. Suppose that we
  perform $T$ iterations of stochastic gradient descent starting from $w^{(1)}
  \in \mathcal{W}$:
  \begin{equation*}
    w^{(t + 1)} = \Pi_w\left( w^{(t)} - \eta_t\mdgrad^{(t)} \right) \eqcomma
  \end{equation*}
  where $\mdgrad^{(t)}$ is a stochastic subgradient of $f_t$, \ie
  $\expectation[\mdgrad^{(t)} \mid \filtration_{t-1}] \in \partial
  f_t(w^{(t)})$, and $\mdgrad^{(t)}$ is $\filtration_t$-measurable.
  If we choose $\eta_t=\frac{1}{\lambda\left(T_0 + t\right)}$ for some
  $T_0\in\N$, then:
  \begin{equation*}
    \frac{1}{T}\sum_{t=1}^T \expectation\left[ f_t\left( w^{(t)} \right) -
    f_t\left(w^{*}\right) \right] \le \frac{\mdgradbound^2 \left(1 + \ln
    T\right)}{2\lambda T} + \frac{\lambda T_0}{2T} \norm{w^{(1)} - w^{*}}_2^2 \eqcomma
  \end{equation*}
  where $w^* \in \mathcal{W}$ is an arbitrary reference vector and $\mdgradbound \ge
  \norm{\mdgrad^{(t)}}_2$ bounds the stochastic subgradient norms for all $t$.
\end{cor}
\begin{prf}{sgd-in-expectation}
  Same proof technique as \corref{mirror-in-expectation}, but based on
  \thmref{sgd} rather than \thmref{mirror}.
\end{prf}

%% file: theorems/cor-sgd-saddle.tex
\begin{cor}{sgd-saddle}
  Let $\alphanorm{\cdot}$ and $\alphadualnorm{\cdot}$ be a norm and its dual.
  Suppose that $\Psi_{\alpha}$ is $1$-strongly convex \wrt $\alphanorm{\cdot}$,
  and has convex conjugate $\Psi_{\alpha}^*$ and associated Bregman divergence
  $\bregman_{\Psi_{\alpha}}$.

  Take $f:\mathcal{W} \times \mathcal{A} \rightarrow \R$ to be $\lambda$-strongly convex in its
  first parameter and concave in its second, let $\filtration$ be a filtration,
  and suppose that we perform $T$ iterations of SGD on $w$ and MD on $\alpha$:
  \begin{align*}
    w^{(t + 1)} = \Pi_w\left( w^{(t)} - \frac{1}{\lambda\left(T_0 + t\right)} \wgrad^{(t)} \right) \eqcomma
    & \hspace{4em}
    \begin{aligned}
      \tilde{\alpha}^{(t+1)} &= \grad \Psi_{\alpha}^* \left( \grad
      \Psi_{\alpha} \left( \alpha^{(t)} \right) + \eta \alphagrad^{(t)}
      \right) \eqcomma \\
      \alpha^{(t+1)} &= \underset{\alpha \in \mathcal{A}}{\argmin}
      \bregman_{\Psi_{\alpha}}\left(\alpha \tallmid
      \tilde{\alpha}^{(t+1)}\right) \eqcomma
    \end{aligned}
  \end{align*}
  where $\wgrad^{(t)}$ is a stochastic subgradient of $f(w^{(t)},
  \alpha^{(t)})$ \wrt its first parameter, and $\alphagrad^{(t)}$ a stochastic
  supergradient \wrt its second, with both $\wgrad^{(t)}$ and
  $\alphagrad^{(t)}$ being $\filtration_t$-measurable.
  Then:
  \begin{align*}
    \MoveEqLeft \frac{1}{T}\sum_{t=1}^T \expectation\left[ f\left( w^{(t)}, \alpha^{*} \right) -
    f\left( w^{*}, \alpha^{(t)} \right) \right] \\
    \le & \frac{\wgradbound^2 \left(1 + \ln
    T\right)}{2\lambda T} + \frac{\lambda T_0}{2T} \norm{w^{(1)} - w^{*}}_2^2 +
    \frac{\bregman_{\Psi_{\alpha}}\left( \alpha^* \tallmid \alpha^{(1)} \right)}{\eta T} +
    \frac{\eta}{2 T} \sum_{t=1}^T \expectation\left[
    \alphadualnorm{\alphagrad^{(t)}}^2 \right] \eqcomma
  \end{align*}
  where $w^* \in \mathcal{W}$ and $\alpha^* \in \mathcal{A}$ are arbitrary
  reference vectors, and $\wgradbound \ge \norm{\wgrad^{(t)}}_2$ bounds the
  stochastic subgradient norms \wrt $w$ for all $t$.
\end{cor}
\begin{prf}{sgd-saddle}
  As we did in the proof of \corref{mirror-saddle}, we will play two convex
  optimization algorithms against each other.
  By \correfs{sgd-in-expectation}{mirror-in-expectation}:
  \begin{align*}
    \frac{1}{T}\sum_{t=1}^T \expectation\left[ f\left( w^{(t)}, \alpha^{(t)}
    \right) - f\left( w^{*}, \alpha^{(t)} \right) \right]
    \le & \frac{\wgradbound^2 \left(1 + \ln T\right)}{2\lambda T} +
    \frac{\lambda T_0}{2T} \norm{w^{(1)} - w^{*}}_2^2 \eqcomma \\
    \frac{1}{T} \sum_{t=1}^T \expectation\left[ f\left( w^{(t)}, \alpha^{*}
    \right) - f\left( w^{(t)}, \alpha^{(t)} \right) \right]
    \le & \frac{\bregman_{\Psi_{\alpha}}\left( \alpha^* \tallmid \alpha^{(1)}
    \right)}{\eta T} + \frac{\eta}{2 T} \sum_{t=1}^T \expectation\left[
    \alphadualnorm{\alphagrad^{(t)}}^2 \right] \eqcomma
  \end{align*}
  Adding these two inequalities gives the claimed result.
\end{prf}

%% file: app-full-light.tex
\section{Analyses of \fullalg and \lightalg}\label{sec:full-light}

We begin by proving that, if $\gamma$ is sufficiently large, then
optimizing the relaxed objective, and projecting the resulting solution,
will bring us close to the optimum of the constrained objective.

\medskip

\input{theorems/lem-projection}
\subsection{Analysis of \fullalg}\label{sec:full-light:full}

We'll now use \lemref{projection} and \corref{mirror-high-probability} to bound
the convergence rate of SGD on the function $h$ of \lemref{projection} (this is
\fullalg).  Like the algorithm itself, the convergence rate is little different
from that found by \citet{MahdaviYaJiYi12} (aside from the bound on
$\norm{\bar{w} - \Pi_g(\bar{w})}_2$), and is included here only for
completeness.

\medskip
\input{theorems/lem-full-suboptimality}
\medskip

In terms of the number of iterations required to achieve some desired level of
suboptimality, this bound on $\fullbound$ may be expressed as:

\medskip
\input{theorems/thm-full}

\subsection{Analysis of \lightalg}\label{sec:full-light:light}

\input{figures/tab-notation-app-light}

Because we use the reduced-variance algorithm of \citet{JohnsonZh13}, and
therefore update the remembered gradient $\memgrad$ one random coordinate at a
time, we must first bound the maximum number of iterations over which a
coordinate can go un-updated:

\medskip
\input{theorems/lem-coupon}
\medskip

We now combine this bound with \corref{mirror-high-probability} and make
appropriate choices of the two {\dgf}s to yield a bound on the \lightalg
convergence rate:

\medskip
\input{theorems/lem-light-suboptimality}
\medskip

In terms of the number of iterations required to achieve some desired level of
suboptimality, this bound on $\lightbound$ and the bound of
\lemref{full-suboptimality} on $\fullbound$ may be combined to yield the
following:

\medskip

\input{theorems/thm-light}

%% file: theorems/lem-full-suboptimality.tex
\begin{lem}{full-suboptimality}
  Suppose that the conditions of \lemref{projection} apply, with $g(w) =
  \max_i(g_i(w))$. Define $\wdiameter \ge \sup_{w,w'\in\mathcal{W}} \norm{w - w'}_2$ as the diameter of
  $\mathcal{W}$, $\fgradbound \ge \norm{\fgrad^{(t)}}_2$ and $\ggradbound \ge
  \norm{\subgrad \max(0, g_i(w))}_2$ as uniform upper bounds on the
  (stochastic) gradient magnitudes of $f$ and the $g_i$s, respectively.

  If we optimize \eqref{constrained-problem} using \algref{full} (\fullname) with
  the step size:
  \begin{equation*}
    \eta = \frac{\wdiameter}{\left(\fgradbound + \gamma \ggradbound\right)
    \sqrt{T}} \eqcomma
  \end{equation*}
  then with probability $1 - \delta$:
  \begin{equation*}
    f\left(\Pi_g\left( \bar{w} \right)\right) - f\left(w^*\right) \le
    \objective\left( \bar{w} \right) - \objective\left(w^*\right) \le
    \fullbound \eqcomma
    \;\;\;\;\;\; \mbox{and} \;\;\;\;\;\;
    \norm{\bar{w} - \Pi_g\left( \bar{w} \right)}_2 \le \frac{\fullbound}{\gamma
    \rho - \flipschitz} \eqcomma
  \end{equation*}
  where $w^* \in \{w \in \mathcal{W} : \forall i . g_i(w) \le 0\}$ is an
  arbitrary constraint-satisfying reference vector, and:
  \begin{equation*}
    \fullbound \le \left(1 + 2\sqrt{2}\right) \wdiameter \left( \fgradbound +
    \gamma \ggradbound \right) \sqrt{1 + \ln\frac{1}{\delta}}
    \sqrt{\frac{1}{T}}
    + \frac{8 \wdiameter \fgradbound \ln\frac{1}{\delta}}{3 T} \eqperiod
  \end{equation*}
\end{lem}
\begin{prf}{full-suboptimality}
  We choose $\Psi(w) = \norm{w}_2^2 / 2$, for which the mirror descent update
  rule is precisely SGD. Because $\Psi_w$ is (half of) the squared Euclidean
  norm, it is trivially $1$-strongly convex \wrt the Euclidean norm, so
  $\mdnorm{\cdot} = \mddualnorm{\cdot} = \norm{\cdot}_2$. Furthermore,
  $\bregman_{\Psi}(w^* \mid w^{(1)}) \le \wdiameter^2 / 2$ and
  $\mdnormradiusstar \le \wdiameter$.

  We may upper bound the $2$-norm of our stochastic gradients as
  $\norm{\wgrad^{(t)}}_2 \le \fgradbound + \gamma \ggradbound$. Only the
  $f$-portion of the objective is stochastic, so the error of the
  $\wgrad^{(t)}$s can be trivially upper bounded, with probability $1$, with
  $\sigma = 2 \fgradbound$. Hence, by \corref{mirror-high-probability} (taking
  $\filtration_t$ to be e.g. the smallest $\sigma$-algebra making
  $\fgrad^{(t)},\dots,\fgrad^{(t)}$ measurable), with probability $1 - \delta$:
  \begin{equation*}
    \frac{1}{T} \sum_{t=1}^T \left( h\left(w^{(t)}\right) - h\left(w^*\right)
    \right)
    \le \frac{\wdiameter^2}{2 \eta T} + \frac{\eta \left( \fgradbound + \gamma
    \ggradbound \right)^2}{2}
    + \frac{2 \sqrt{2} \wdiameter \fgradbound
    \sqrt{\ln\frac{1}{\delta}}}{\sqrt{T}} + \frac{8 \wdiameter \fgradbound
    \ln\frac{1}{\delta}}{3 T} \eqperiod
  \end{equation*}
  Plugging in the definition of $\eta$, moving the average defining $\bar{w}$
  inside $h$ by Jensen's inequality, substituting $f(w^*) = h(w^*)$ because
  $w^*$ satisfies the constraints, applying \lemref{projection} and simplifying
  yields the claimed result.
\end{prf}

%% file: theorems/thm-full.tex
\begin{thm}{full}
  \switchshowproofs{
    Suppose that the conditions of \lemrefs{projection}{full-suboptimality}
    apply, and that $\eta$ is as defined in \lemref{full-suboptimality}.
  }{
    Suppose that the conditions of \lemref{projection} apply, with $g(w) =
    \max_i(g_i(w))$. Define $\wdiameter \ge \norm{w - w'}_2$ as the diameter of
    $\mathcal{W}$, $\fgradbound \ge \norm{\fgrad^{(t)}}_2$ and $\ggradbound \ge
    \norm{\subgrad \max(0, g_i(w))}_2$ as uniform upper bounds on the
    (stochastic) gradient magnitudes of $f$ and the $g_i$s, respectively, for
    all $i \in \{1,\dots,m\}$ and $w,w'\in\mathcal{W}$.
  }

  If we optimize \eqref{constrained-problem} using $T_{\epsilon}$ iterations of
  \algref{full} (\fullname):
  \begin{align*}
    \MoveEqLeft T_{\epsilon} =
    O\left( \frac{\wdiameter^2 \left(\fgradbound + \gamma \ggradbound\right)^2
    \ln\frac{1}{\delta}}{\epsilon^2} \right) \eqcomma
  \end{align*}
  \switchshowproofs{
    then $\fullbound \le \epsilon$ with probability $1 - \delta$.
  }{
    with the step size:
    \begin{equation*}
      \eta = \frac{\wdiameter}{\left(\fgradbound + \gamma \ggradbound\right) \sqrt{T}} \eqcomma
    \end{equation*}
    then with probability $1 - \delta$:
    \begin{equation*}
      f\left(\Pi_g\left( \bar{w} \right)\right) - f\left(w^*\right) \le
      \objective\left( \bar{w} \right) - \objective\left(w^*\right) \le
      \epsilon
      \;\;\;\;\;\; \mbox{and} \;\;\;\;\;\;
      \norm{\bar{w} - \Pi_g\left( \bar{w} \right)}_2 \le \frac{epsilon}{\gamma
      \rho - \flipschitz} \eqcomma
    \end{equation*}
  }
  where $w^* \in \{w \in \mathcal{W} : \forall i . g_i(w) \le 0\}$ is an
  arbitrary constraint-satisfying reference vector.
\end{thm}
\begin{prf}{full}
  Based on the bound of \lemref{full-suboptimality}, define:
  \begin{align*}
    x =& \sqrt{T} \eqcomma \\
    c =& \frac{8}{3} \wdiameter \fgradbound \ln\frac{1}{\delta} \eqcomma \\
    b =& \left(1 + 2\sqrt{2}\right) \wdiameter \left(\fgradbound + \gamma
    \ggradbound\right) \sqrt{1 + \ln\frac{1}{\delta}} \eqcomma \\
    a =& -\epsilon \eqcomma
  \end{align*}
  and consider the polynomial $0 = a x^2 + bx + c$. Roots of this polynomial
  are $x$s for which $\fullbound = \epsilon$, while for $x$s larger than any
  root we'll have that $\fullbound \le \epsilon$. Hence, we can bound the $T$
  required to ensure $\epsilon$-suboptimality by bounding the roots of this
  polynomial. By the Fujiwara bound~\citep{WikipediaPolynomialRoots}:
  \begin{equation}
    \label{eq:full:final} T_\epsilon \le
    \max\left\{ \frac{4 \left(9 + 4\sqrt{2}\right)
    \wdiameter^2 \left(\fgradbound + \gamma \ggradbound\right)^2 \left(1 +
    \ln\frac{1}{\delta}\right)}{\epsilon^2},
    \frac{16 \wdiameter \fgradbound
    \ln\frac{1}{\delta}}{3\epsilon} \right\}
    \eqcomma
  \end{equation}
  giving the claimed result.
\end{prf}

%% file: figures/tab-notation-app-light.tex
\begin{table*}[t]

\caption{New notation in \appref{full-light:light}.}

\label{tab:notation-app-light}

\begin{center}

\begin{tabular}{lll}
  \hline
  \textbf{Symbol} & \textbf{Description} & \textbf{Definition} \\
  \hline
  $\wnorm{\cdot}$, $\wdualnorm{\cdot}$ & Norm on $\mathcal{W}$ and its dual & $\wnorm{\cdot} = \wdualnorm{\cdot} = \norm{\cdot}_2$ \\
  $\pnorm{\cdot}$, $\pdualnorm{\cdot}$ & Norm on $\Delta^m$ and its dual & $\pnorm{\cdot} = \norm{\cdot}_1$, $\pdualnorm{\cdot} = \norm{\cdot}_{\infty}$ \\
  $\Psi_w, \Psi_w^*$ & A \dgf on $\mathcal{W}$ and its convex conjugate & $\Psi_w(w) = \norm{w}_2^2 / 2$ \\
  $\Psi_p, \Psi_p^*$ & A \dgf on $\Delta^m$ and its convex conjugate & $\Psi_p(p) = \sum_{i=1}^m p_i \ln p_i$ \\
  $\wnormradiusstar$ & Bound on $w^*$-centered radius of $\mathcal{W}$ & $\wnormradiusstar = \wdiameter \ge \norm{w - w^*}_2$ \\
  $\pnormradiusstar$ & Bound on $p^*$-centered radius of $\Delta^m$ & $\pnormradiusstar = 1 \ge \norm{p - p^*}_1$ \\
  $\sigma_w$ & Bound on $\wgrad$ error & $\sigma_w = \fgradbound + \gamma \ggradbound$ \\
  $\sigma_p$ & Bound on $\pgrad$ error & $\sigma_p \ge \norm{\expectation[\pgrad^{(t)} \mid \filtration_{t-1}] - \pgrad^{(t)}}_\infty$ \\
  $1 - \delta_{\sigma w}$ & Probability that $\sigma_w$ bound holds & $1 - \delta_{\sigma w} = 1$ \\
  $1 - \delta_{\sigma p}$ & Probability that $\sigma_p$ bound holds & \\
  \hline
\end{tabular}

\end{center}

\end{table*}

%% file: theorems/lem-coupon.tex
\begin{lem}{coupon}
  Consider a process which maintains a sequence of vectors $s^{(t)} \in \N^m$
  for $t \in \{1, \dots, T\}$, where $s^{(1)}$ is initialized to zero and
  $s^{(t+1)}$ is derived from $s^{(t)}$ by independently sampling $k =
  \abs{S_t} \le m$ random indices $S_t \subseteq \{1,\dots,m\}$ uniformly
  without replacement, and then setting $s^{(t+1)}_{j} = t$ for $j \in S_t$ and
  $s^{(t+1)}_j = s^{(t)}_j$ for $j \notin S_t$.
  Then, with probability $1-\delta$:
  \begin{equation*}
    \max_{t, j} \left( t - s^{(t)}_j \right) \le 1 + \frac{2 m}{k} \ln \left(
    \frac{2 m T}{\delta} \right) \eqperiod
  \end{equation*}
\end{lem}
\begin{prf}{coupon}
  This is closely related to the ``coupon collector's
  problem''~\citep{WikipediaCoupon}. We will begin by partitioning time into
  contiguous size-$n$ chunks, with $1,\dots,n$ forming the first chunk,
  $n+1,\dots,2n$ the second, and so on.

  Within each chunk the probability that any particular index was never sampled
  is $((m-k)/m)^n$, so by the union bound the probability that any one of the
  $m$ indices was never sampled is bounded by $m((m-k)/m)^n$:
  \begin{equation*}
    m\left(\frac{m-k}{m}\right)^n
    \le \exp\left( \ln m + n \ln\left(\frac{m-k}{m}\right) \right)
    \le \exp\left( \ln m - \frac{n k}{m} \right) \eqperiod
  \end{equation*}
  Define $n = \lceil (m / k) \ln(2 m T / \delta) \rceil$, so:
  \begin{equation*}
    m\left(\frac{m-k}{m}\right)^n
    \le \exp\left( \ln m - \ln\left( \frac{2 m T}{\delta} \right) \right)
    \le \frac{\delta}{2 T}
    \eqperiod
  \end{equation*}
  This shows that for this choice of $n$, the probability of there existing an
  index which is never sampled in some particular batch is bounded by $\delta /
  2 T$.  By the union bound, the probability of \emph{any} of $\lceil T / n
  \rceil$ batches containing an index which is never sampled is bounded by
  $(\delta / 2 T) \lceil T / n \rceil \le (\delta / 2 n) + (\delta / 2 T) \le
  \delta$.

  If every index is sampled within every batch, then over the first $n \lceil T
  / n \rceil \ge T$ steps, the most steps which could elapse over which a
  particular index is not sampled is $2n-2$ (if the index is sampled on the
  first step of one chunk, and the last step of the next chunk), which implies
  the claimed result.
\end{prf}

%% file: theorems/lem-light-suboptimality.tex
\begin{lem}{light-suboptimality}
  Suppose that the conditions of \lemref{projection} apply, with $g(w) =
  \max_i(g_i(w))$. Define $\wdiameter \ge \sup_{w,w'\in\mathcal{W}} \max\{1,
  \norm{w - w'}_2\}$ as a bound on the diameter of $\mathcal{W}$ (notice that
  we also choose $\wdiameter$ to be at least $1$), $\fgradbound \ge
  \norm{\fgrad^{(t)}}_2$ and $\ggradbound \ge \norm{\subgrad \max(0,
  g_i(w))}_2$ as uniform upper bounds on the (stochastic) gradient magnitudes
  of $f$ and the $g_i$s, respectively, for all $i \in \{1,\dots,m\}$.
  We also assume that all $g_i$s are $\glipschitz$-Lipschitz \wrt
  $\norm{\cdot}_2$, \ie $\abs{g_i(w) - g_i(w')} \le \glipschitz \norm{w -
  w'}_2$ for all $w,w'\in\mathcal{W}$.

  Define:
  \begin{equation*}
    k = \left\lceil \frac{m \left(1 + \ln
    m\right)^{\nicefrac{3}{4}} \sqrt{1 + \ln\frac{1}{\delta}} \sqrt{1 + \ln
    T}}{T^{\nicefrac{1}{4}}} \right\rceil
    \eqperiod
  \end{equation*}
  If $k \le m$ and we optimize \eqref{constrained-problem} using \algref{light}
  (\lightname), basing the stochastic gradients \wrt $p$ on $k$ constraints at
  each iteration, and using the step size:
  \begin{equation*}
    \eta = \frac{\sqrt{1 + \ln m} \wdiameter}{\left( \fgradbound + \gamma
    \ggradbound + \gamma \glipschitz \wdiameter\right) \sqrt{T}}
    \eqcomma
  \end{equation*}
  then it holds with probability $1 - \delta$ that:
  \begin{equation*}
    f\left(\Pi_g\left( \bar{w} \right)\right) - f\left(w^*\right) \le
    \objective\left( \bar{w} \right) - \objective\left(w^*\right) \le
    \lightbound \eqcomma
    \;\;\;\;\;\; \mbox{and} \;\;\;\;\;\;
    \norm{\bar{w} - \Pi_g\left( \bar{w} \right)}_2 \le
    \frac{\lightbound}{\gamma \rho - \flipschitz} \eqcomma
  \end{equation*}
  where $w^* \in \{w \in \mathcal{W} : \forall i . g_i(w) \le 0\}$ is an
  arbitrary constraint-satisfying reference vector, and:
  \begin{equation*}
    \lightbound \le 67 \sqrt{1 + \ln m} \wdiameter \left(\fgradbound + \gamma
    \ggradbound + \gamma \glipschitz \wdiameter\right) \sqrt{1 +
    \ln\frac{1}{\delta}} \sqrt{\frac{1}{T}}
    \eqperiod
  \end{equation*}
  If $k > m$, then we should fall-back to using \fullalg, in which case the
  result of \lemref{full-suboptimality} will apply.
\end{lem}
\begin{prf}{light-suboptimality}
  We choose $\Psi_w(w) = \norm{w}_2^2 / 2$ and $\Psi_p(p) = \sum_{i=1}^m p_i
  \ln p_i$ to be the squared Euclidean norm divided by $2$ and the negative
  Shannon entropy, respectively, which yields the updates of \algref{light}.
  We assume that the $\fgrad^{(t)}$s are random variables on some probability
  space (depending on the source of the stochastic gradients of $f$), and
  likewise the $i_t$s and $j_t$s on another, so $\filtration_t$ may be taken to
  be the product of the smallest $\sigma$-algebras which make
  $\fgrad^{(1)},\dots,\fgrad^{(t)}$ and $i_1,j_1,\dots,i_t,j_t$ measurable,
  respectively, with conditional expectations being taken \wrt the product
  measure.
  Under the definitions of \corref{mirror-saddle} (taking $\alpha = p$), with
  probability $1 - \delta_{\sigma w} - \delta_{\sigma p} - 2 \delta'$:
  \begin{align*}
    \MoveEqLeft \frac{1}{T} \sum_{t=1}^T \relaxedobjective\left(w^{(t)},
    p^*\right) - \frac{1}{T} \sum_{t=1}^T \relaxedobjective\left(w^*,
    p^{(t)}\right) \\
    \le & \frac{\bregman_{\Psi_w}\left(w^* \tallmid w^{(1)}\right) +
    \bregman_{\Psi_p}\left(p^* \tallmid p^{(1)}\right)}{\eta T} +
    \frac{\eta}{2 T} \sum_{t=1}^T \left( \wdualnorm{\wgrad^{(t)}}^2 +
    \pdualnorm{\pgrad^{(t)}}^2 \right) \\
    & + \frac{\sqrt{2} \left( \wnormradiusstar \sigma_w + \pnormradiusstar
    \sigma_p \right) \sqrt{\ln\frac{1}{\delta'}}}{\sqrt{T}} +
    \frac{4 \left( \wnormradiusstar \sigma_w + \pnormradiusstar \sigma_p
    \right) \ln\frac{1}{\delta'}}{3 T} \eqperiod
  \end{align*}
  As in the proof of \lemref{full-suboptimality}, $\Psi_w$ is $1$-strongly
  convex \wrt the Euclidean norm, so $\wnorm{\cdot} = \wdualnorm{\cdot} =
  \norm{\cdot}_2$, $\bregman_{\Psi_w}(w^* \mid w^{(1)}) \le \wdiameter^2 / 2$
  and $\wnormradiusstar \le \wdiameter$.
  Because $\Psi_p$ is the negative entropy, which is $1$-strongly convex \wrt
  the $1$-norm (this is Pinsker's inequality), $\pnorm{\cdot} = \norm{\cdot}_1$
  and $\pdualnorm{\cdot} = \norm{\cdot}_\infty$, implying that
  $\pnormradiusstar = 1$. Since $p^{(1)}$ is initialized to the uniform
  distribution, $\bregman_{\Psi_p}(p^* \mid p^{(1)}) = D_{KL}(p^* \mid p^{(1)})
  \le \ln m$.

  The stochastic gradient definitions of \algref{light} give that
  $\wdualnorm{\wgrad^{(t)}} \le \fgradbound + \gamma \ggradbound$ and $\sigma_w
  \le 2(\fgradbound + \gamma \ggradbound)$ with probability $1 = 1 -
  \delta_{\sigma w}$ by the triangle inequality, and $\relaxedobjective(w^*,
  p^{(t)}) = f(w^*)$ because $w^*$ satisfies the constraints. All of these
  facts together give that, with probability $1 - \delta_{\sigma p} - \delta'$:
  \begin{align*}
    \MoveEqLeft \frac{1}{T} \sum_{t=1}^T \relaxedobjective\left(w^{(t)},
    p^*\right) - f\left(w^*\right) \\
    \le & \frac{\wdiameter^2 + 2 \ln m}{2 \eta T} +
    \frac{\eta}{2 T} \sum_{t=1}^T \left( \left( \fgradbound + \gamma
    \ggradbound \right)^2 + \norm{\pgrad^{(t)}}_\infty^2 \right) \\
    & + \frac{\sqrt{2} \left( 2 \wdiameter (\fgradbound + \gamma \ggradbound) +
    \sigma_p \right) \sqrt{\ln\frac{1}{\delta'}}}{\sqrt{T}} +
    \frac{4 \left( 2 \wdiameter \left( \fgradbound + \gamma \ggradbound \right) + \sigma_p
    \right) \ln\frac{1}{\delta'}}{3 T} \eqperiod
  \end{align*}
  We now move the average defining $\bar{w}$ inside $\relaxedobjective$ (which
  is convex in its first parameter) by Jensen's inequality, and use the fact
  that there exists a $p^*$ such that $\relaxedobjective(w, p^*) = \objective(w)$ to
  apply \lemref{projection}:
  \begin{align}
    \label{eq:light:intermediate1} \MoveEqLeft \lightbound
    \le \frac{\wdiameter^2 + 2 \ln m}{2 \eta T} +
    \frac{\eta}{2 T} \sum_{t=1}^T \left( \left( \fgradbound + \gamma
    \ggradbound \right)^2 + \norm{\pgrad^{(t)}}_\infty^2 \right) \\
    \notag & + \frac{\sqrt{2} \left( 2 \wdiameter (\fgradbound + \gamma
    \ggradbound) + \sigma_p \right) \sqrt{\ln\frac{1}{\delta'}}}{\sqrt{T}} +
    \frac{4 \left( 2 \wdiameter \left( \fgradbound + \gamma \ggradbound \right) + \sigma_p
    \right) \ln\frac{1}{\delta'}}{3 T} \eqperiod
  \end{align}
  By the triangle inequality and the fact that $(a+b)^2 \le 2a^2 + 2b^2$:
  \begin{align*}
    \norm{\pgrad^{(t)}}_\infty^2
    & \le 2 \norm{\expectation\left[ \pgrad^{(t)} \tallmid \filtration_{t-1}
    \right]}_\infty^2 + 2 \norm{\expectation\left[ \pgrad^{(t)} \tallmid
    \filtration_{t-1} \right] - \pgrad^{(t)}}_\infty^2 \\
    & \le 2 \gamma^2 \glipschitz^2 \wdiameter^2 + 2 \norm{\expectation\left[
    \pgrad^{(t)} \tallmid \filtration_{t-1} \right] - \pgrad^{(t)}}_\infty^2 \\
    & \le 2 \gamma^2 \glipschitz^2 \wdiameter^2 + 2 \sigma_p^2 \eqperiod
  \end{align*}
  Substituting into \eqref{light:intermediate1} and using the fact that $a + b
  \le (\sqrt{a} + \sqrt{b})^2$:
  \begin{align}
    \label{eq:light:intermediate2} \MoveEqLeft \lightbound
    \le \frac{\wdiameter^2 + 2 \ln m}{2 \eta T} +
    \frac{\eta}{2} \left( \fgradbound + \gamma \ggradbound + \sqrt{2} \gamma
    \glipschitz \wdiameter \right)^2 +
    \eta \sigma_p^2 \\
    \notag & + \frac{\sqrt{2} \left( 2 \wdiameter (\fgradbound + \gamma
    \ggradbound) + \sigma_p \right) \sqrt{\ln\frac{1}{\delta'}}}{\sqrt{T}} +
    \frac{4 \left( 2 \wdiameter \left( \fgradbound + \gamma \ggradbound
    \right) + \sigma_p \right) \ln\frac{1}{\delta'}}{3 T} \eqperiod
  \end{align}
  We will now turn our attention to the problem of bounding $\sigma_p$.
  Notice that because we sample \iid $j_t$s uniformly at every iteration, they
  form an instance of the process of \lemref{coupon} with $\memgrad_j^{(t)} =
  \max(0, g_j( w^{(s_j^{(t)})} ))$, showing that with probability $1 -
  \delta_{\sigma p}$:
  \begin{equation}
    \label{eq:light:coupon} \max_{t,j} \left( t - s_j^{(t)} \right) \le
    1 + \frac{2 m}{k} \ln\left( \frac{2 m T}{\delta_{\sigma p}} \right)
    \eqperiod
  \end{equation}
  By the definition of $\pgrad^{(t)}$ (\algref{light}):
  \begin{align*}
    \MoveEqLeft \norm{\expectation\left[ \pgrad^{(t)} \tallmid
    \filtration_{t-1} \right] - \pgrad^{(t)}}_\infty^2 \\
    = & \gamma^2 \norm{ \left( \sum_{j=1}^m e_j \max\left\{0, g_j\left( w^{(t)}
    \right)\right\} - \memgrad^{(t)} \right) - \frac{m}{k} \sum_{j \in S_t} \left( e_{j} \max\left\{0,
    g_{j}\left( w^{(t)} \right)\right\} - e_{j} \memgrad_{j}^{(t)} \right)
    }_\infty^2 \\
    \le & \gamma^2 \left(\frac{m-k}{k}\right)^2 \max_j \left( \max\left\{0, g_j\left(
    w^{(t)} \right)\right\} - \memgrad_j^{(t)} \right)^2 \\
    \le & \gamma^2 \left(\frac{m-k}{k}\right)^2 \glipschitz^2 \norm{ w^{(t)} -
    w^{(s_j^{(t)})} }_2^2 \\
    \le & \gamma^2 \left(\frac{m-k}{k}\right)^2 \glipschitz^2 \eta^2 \left( \fgradbound +
    \gamma \ggradbound \right)^2 \left( t - s_j^{(t)} \right)^2 \\
    \le & \gamma^2 \left(\frac{m-k}{k}\right)^2 \glipschitz^2 \eta^2 \left( \fgradbound +
    \gamma \ggradbound \right)^2 \left( 1 + \frac{2 m}{k} \ln\left( \frac{2
    m T}{\delta_{\sigma p}} \right) \right)^2 \\
    \le & 6 \gamma^2 \left(\frac{m}{k}\right)^4 \glipschitz^2 \eta^2 \left( \fgradbound +
    \gamma \ggradbound \right)^2 \left( 1 + \ln\left( \frac{m T}{\delta_{\sigma p}} \right) \right)^2 \\
  \end{align*}
  where in the second step we used the definition of the $\infty$-norm, in the
  third we used the Lipschitz continuity of the $g_i$s (and hence of their
  positive parts), in the fourth we bounded the distance between two iterates
  with the number of iterations times a bound on the total step size, and in
  the fifth we used \eqref{light:coupon}. This shows that we may define:
  \begin{equation*}
    \sigma_p = \sqrt{6} \gamma \left(\frac{m}{k}\right)^2 \glipschitz \eta \left( \fgradbound + \gamma
    \ggradbound \right) \left( 1 + \ln\left( \frac{m T}{\delta_{\sigma p}}
    \right) \right) \eqcomma
  \end{equation*}
  and it will satisfy the conditions of \corref{mirror-saddle}. Notice that,
  due to the $\eta$ factor, $\sigma_p$ will be \emph{decreasing} in $T$.
  Substituting the definitions of $\eta$ and $\sigma_p$ into
  \eqref{light:intermediate2}, choosing $\delta_{\sigma p} = \delta' = \delta/3$ and using the assumption that $\wdiameter \ge 1$ gives
  that with probability $1 - \delta$:
  \begin{align*}
    \MoveEqLeft \lightbound
    \le 2 \left(1 + \sqrt{2}\right) \sqrt{1 + \ln 3} \sqrt{1 + \ln m} \wdiameter \left(\fgradbound + \gamma \ggradbound + \gamma \glipschitz \wdiameter\right) \sqrt{1 + \ln\frac{1}{\delta}} \left( \frac{1}{\sqrt{T}} \right) \\
    & + \left(2\sqrt{3} + \frac{8}{3}\right) \left(1 + \ln 3\right)^{\nicefrac{3}{2}} \left(\frac{m}{k}\right)^2 \left(1 + \ln m\right)^{\nicefrac{3}{2}} \wdiameter \left(\fgradbound + \gamma \ggradbound\right) \left(1 + \ln\frac{1}{\delta}\right)^{\nicefrac{3}{2}} \left(\frac{1 + \ln T}{T}\right) \\
    & + 2 \left(3 + 2 \sqrt{\frac{2}{3}}\right) \left(1 + \ln 3\right)^2 \left(\frac{m}{k}\right)^4 \left(1 + \ln m\right)^{\nicefrac{7}{2}} \wdiameter \left(\fgradbound + \gamma \ggradbound\right) \left(1 + \ln\frac{1}{\delta}\right)^2 \left(\frac{\left(1 + \ln T\right)^2}{T^{\nicefrac{3}{2}}}\right)
    \eqperiod
  \end{align*}
  Rounding up the constant terms:
  \begin{align*}
    \MoveEqLeft \lightbound
    \le 7 \sqrt{1 + \ln m} \wdiameter \left(\fgradbound + \gamma \ggradbound + \gamma \glipschitz \wdiameter\right) \sqrt{1 + \ln\frac{1}{\delta}} \left( \frac{1}{\sqrt{T}} \right) \\
    & + 19 \left(\frac{m}{k}\right)^2 \left(1 + \ln m\right)^{\nicefrac{3}{2}} \wdiameter \left(\fgradbound + \gamma \ggradbound\right) \left(1 + \ln\frac{1}{\delta}\right)^{\nicefrac{3}{2}} \left(\frac{1 + \ln T}{T}\right) \\
    & + 41 \left(\frac{m}{k}\right)^4 \left(1 + \ln m\right)^{\nicefrac{7}{2}} \wdiameter \left(\fgradbound + \gamma \ggradbound\right) \left(1 + \ln\frac{1}{\delta}\right)^2 \left(\frac{\left(1 + \ln T\right)^2}{T^{\nicefrac{3}{2}}}\right)
    \eqperiod
  \end{align*}
  Substituting the definition of $k$, simplifying and bounding yields the claimed result.
\end{prf}

%% file: app-mid.tex
\section{Analysis of \midalg}\label{sec:mid}

We now move on to the analysis of our \lightalg variant for $\lambda$-strongly
convex objectives, \algref{mid} (\midalg). While we were able to prove a
high-probability bound for \lightalg, we were unable to do so for \midalg,
because the extra terms resulting from the use of a Bernstein-type martingale
inequality were too large (since the other terms shrank as a result of the
strong convexity assumption). Instead, we give an in-expectation result, and
leave the proof of a corresponding high-probability bound to future work.

Our first result is an analogue of
\lemrefs{full-suboptimality}{light-suboptimality}, and bounds the suboptimality
achieved by \midalg as a function of the iteration counts $T_1$ and $T_2$ of
the two phases:

\medskip
\input{theorems/lem-mid-suboptimality}
\medskip

We now move on to the main result: a bound on the number of iterations
(equivalently, the number of stochastic loss gradients) and constraint checks
required to achieve $\epsilon$-suboptimality:

\medskip

\input{theorems/thm-mid}

%% file: theorems/lem-mid-suboptimality.tex
\begin{lem}{mid-suboptimality}
  Suppose that the conditions of \lemref{projection} apply, with $g(w) =
  \max_i(g_i(w))$. Define $\fgradbound \ge \norm{\fgrad^{(t)}}_2$ and
  $\ggradbound \ge \norm{\subgrad \max(0, g_i(w))}_2$ as uniform upper bounds
  on the (stochastic) gradient magnitudes of $f$ and the $g_i$s, respectively,
  for all $i \in \{1,\dots,m\}$.
  We also assume that $f$ is $\lambda$-strongly convex, and that all $g_i$s are
  $\glipschitz$-Lipschitz \wrt $\norm{\cdot}_2$, \ie $\abs{g_i(w) - g_i(w')}
  \le \glipschitz \norm{w - w'}_2$ for all $w,w'\in\mathcal{W}$.

  If we optimize \eqref{constrained-problem} using \algref{mid} (\midname) with
  the $p$-update step size $\eta = \lambda / 2 \gamma^2 \glipschitz^2$, then:
  \begin{align*}
    \MoveEqLeft \expectation\left[ \norm{\Pi_g(\bar{w}) - w^*}_2^2 \right] \le
    \expectation\left[ \norm{\bar{w} - w^*}_2^2 \right] \\
    & \le \frac{2 \left(\fgradbound + \gamma \ggradbound\right)^2 \left(2 + \ln
    T_1 + \ln T_2\right) + 8 \gamma^2 \glipschitz^2 \ln m}{\lambda^2 T_2} +
    \frac{3 m^4 \left(1 + \ln m\right)^2 \left(\fgradbound + \gamma
    \ggradbound\right)^2}{\lambda^2 T_1^2}
    \eqcomma
  \end{align*}
  where $w^* = \argmin_{\{w \in \mathcal{W} : \forall i . g_i(w) \le 0\}} f(w)$
  is the \emph{optimal} constraint-satisfying reference vector.
\end{lem}
\begin{prf}{mid-suboptimality}
  As in the proof of \lemref{full-suboptimality}, the first phase of
  \algref{mid} is nothing but (strongly convex) SGD on the overall objective
  function $h$, so by \corref{sgd-in-expectation}:
  \begin{equation*} \frac{1}{T_1}\sum_{t=1}^{T_1} \expectation\left[
  \objective\left( w^{(t)} \right) - \objective\left(w^{*}\right) \right] \le
  \frac{\wgradbound^2 \left(1 + \ln T_1\right)}{2\lambda T_1} \eqcomma
  \end{equation*}
  so by Jensen's inequality:
  \begin{equation} \label{eq:mid:intermediate1} \expectation\left[
  \objective\left( w^{(T_1 + 1)} \right) - \objective\left(w^{*}\right) \right]
  \le \frac{\wgradbound^2 \left(1 + \ln T_1\right)}{2\lambda T_1} \eqperiod
  \end{equation}
  For the second phase, as in the proof of \lemref{light-suboptimality}, we
  choose $\Psi_p(p) = \sum_{i=1}^m p_i \ln p_i$ to be negative Shannon entropy,
  which yields the second-phase updates of \algref{mid}. By
  \corref{sgd-saddle}:
  \begin{align*}
    \MoveEqLeft \frac{1}{T_2}\sum_{t=T_1 + 1}^{T_2} \expectation\left[
    \relaxedobjective\left( w^{(t)}, p^{*} \right) - \relaxedobjective\left(
    w^{*}, p^{(t)} \right) \right] \\
    \le & \frac{\wgradbound^2 \left(1 + \ln T\right)}{2\lambda T_2} +
    \frac{\lambda T_1}{2 T_2} \norm{w^{(T_1 + 1)} - w^{*}}_2^2 +
    \frac{\bregman_{\Psi_p}\left( p^* \tallmid p^{(T_1 + 1)} \right)}{\eta T_2}
    + \frac{\eta}{2 T_2} \sum_{t=T_1 + 1}^{T_2} \expectation\left[
    \pdualnorm{\pgrad^{(t)}}^2 \right] \eqperiod
  \end{align*}
  As before, $\pnorm{\cdot} = \norm{\cdot}_1$, $\pdualnorm{\cdot} =
  \norm{\cdot}_\infty$, and $\bregman_{\Psi_p}(p^* \mid p^{(T_1 + 1)}) =
  D_{KL}(p^* \mid p^{(T_1 + 1)}) \le \ln m$. Hence:
  \begin{align*}
    \MoveEqLeft \frac{1}{T_2}\sum_{t=T_1 + 1}^{T_2} \expectation\left[
    \relaxedobjective\left( w^{(t)}, p^{*} \right) - \relaxedobjective\left(
    w^{*}, p^{(t)} \right) \right] \\
    \le & \frac{\wgradbound^2 \left(1 + \ln T_2\right)}{2\lambda T_2} +
    \frac{\lambda T_1}{2 T_2} \norm{w^{(T_1 + 1)} - w^{*}}_2^2 +
    \frac{\ln m}{\eta T_2} + \frac{\eta}{2 T_2} \sum_{t=T_1 + 1}^{T_2}
    \expectation\left[ \norm{\pgrad^{(t)}}_{\infty}^2 \right] \eqperiod
  \end{align*}
  Since $h$ is $\lambda$-strongly convex and $w^*$ is optimal, $\norm{w^{(T_1 +
  1)} - w^{*}}_2^2 \le \frac{2}{\lambda}( \objective( w^{(T_1 + 1)} ) -
  \objective( w^* ) )$. By \eqref{mid:intermediate1}:
  \begin{align*}
    \MoveEqLeft \frac{1}{T_2}\sum_{t=T_1 + 1}^{T_2} \expectation\left[
    \relaxedobjective\left( w^{(t)}, p^{*} \right) - \relaxedobjective\left(
    w^{*}, p^{(t)} \right) \right] \\
    \le & \frac{\wgradbound^2 \left(2 + \ln T_1 + \ln T_2\right)}{2\lambda T_2}
    +
    \frac{\ln m}{\eta T_2} + \frac{\eta}{2 T_2} \sum_{t=T_1 + 1}^{T_2}
    \expectation\left[ \norm{\pgrad^{(t)}}_{\infty}^2 \right] \eqperiod
  \end{align*}
  Since the (uncentered) second moment is equal to the mean plus the variance,
  and using the fact that $\relaxedobjective( w^{*}, p^{(t)} ) = f( w^{*} )$
  since all constraints are satisfied at $w^{*}$:
  \begin{align}
    \label{eq:mid:intermediate2} \MoveEqLeft \frac{1}{T_2}\sum_{t=T_1 +
    1}^{T_2} \expectation\left[ \relaxedobjective\left( w^{(t)}, p^{*} \right)
    \right] - f\left( w^{*} \right) \\
    \notag \le & \frac{\wgradbound^2 \left(2 + \ln T_1 + \ln
    T_2\right)}{2\lambda T_2} +
    \frac{\ln m}{\eta T_2} + \frac{\eta}{2 T_2} \sum_{t=T_1 + 1}^{T_2} \left(
    \expectation\left[ \norm{\pgrad^{(t)}}_{\infty} \right] \right)^2 +
    \frac{\eta \sigma_p^2}{2}
    \eqcomma
  \end{align}
  where $\sigma_p^2$ is the variance of $\norm{\pgrad^{(t)}}_{\infty}$. Next
  observe that:
  \begin{align*}
    \left( \expectation\left[ \norm{\pgrad^{(t)}}_{\infty} \right] \right)^2 =&
    \left( \expectation \left[ \max_{j\in\{1,\dots,m\}} \gamma \max\left\{0,
    g_j\left(w^{(t)}\right)\right\} \right] \right)^2 \\
    \le& \gamma^2 \glipschitz^2 \expectation\left[ \norm{w^{(t)} - w^*}_2^2
    \right] \\
    \le& \frac{2 \gamma^2 \glipschitz^2}{\lambda} \expectation\left[
    \relaxedobjective\left( w^{(t)}, p^{*} \right) - \relaxedobjective\left(
    w^{*}, p^{*} \right) \right] \eqcomma
  \end{align*}
  the first step using the fact that the $g_j$s are $\glipschitz$-Lipschitz and
  Jensen's inequality. For the second step, we choose $p^*$ such that $w^*,
  p^*$ is a minimax optimal pair (recall that $w^*$ is optimal by assumption),
  and use the $\lambda$-strong convexity of $\relaxedobjective$. Substituting
  into \eqref{mid:intermediate2} and using the fact that
  $\relaxedobjective(w^*, p^*) = f(w^*)$:
  \begin{equation*}
    \left( 1 - \frac{\eta \gamma^2 \glipschitz^2}{\lambda} \right) \left(
    \frac{1}{T_2} \sum_{t=T_1 + 1}^{T_2} \expectation\left[
    \relaxedobjective\left( w^{(t)}, p^{*} \right) \right] - f\left( w^{*}
    \right) \right)
    \le \frac{\wgradbound^2 \left(2 + \ln T_1 + \ln T_2\right)}{2\lambda T_2} +
    \frac{\ln m}{\eta T_2} +
    \frac{\eta \sigma_p^2}{2}
    \eqperiod
  \end{equation*}
  Substituting $\eta=\lambda / 2 \gamma^2 \glipschitz^2$ and using Jensen's
  inequality:
  \begin{equation}
    \label{eq:mid:intermediate3} \expectation\left[ \relaxedobjective\left(
    \bar{w}, p^{*} \right) \right] - f\left( w^{*} \right)
    \le \frac{\wgradbound^2 \left(2 + \ln T_1 + \ln T_2\right)}{\lambda T_2} +
    \frac{4 \gamma^2 \glipschitz^2 \ln m}{\lambda T_2} +
    \frac{\lambda \sigma_p^2}{2 \gamma^2 \glipschitz^2}
    \eqperiod
  \end{equation}
  We now follow the proof of \lemref{light-suboptimality} and bound
  $\sigma_p^2$. By the definition of $\pgrad^{(t)}$ (\algref{mid}):
  \begin{align*}
    \sigma_p^2 =& \expectation\left[ \norm{\expectation\left[ \pgrad^{(t)}
    \tallmid \filtration_{t-1} \right] - \pgrad^{(t)}}_\infty^2 \right] \\
    = & \gamma^2 \expectation\left[ \norm{ \left( \sum_{j=1}^m e_j
    \max\left\{0, g_j\left( w^{(t)} \right)\right\} - \memgrad^{(t)} \right) -
    m \left( e_{j_t} \max\left\{0, g_{j_t}\left( w^{(t)} \right)\right\} -
    e_{j_t} \memgrad_{j_t}^{(t)} \right) }_\infty^2 \right] \\
    \le & \gamma^2 \left(m-1\right)^2 \expectation\left[ \max_j \left(
    \max\left\{0, g_j\left( w^{(t)} \right)\right\} - \memgrad_j^{(t)}
    \right)^2 \right] \eqperiod
  \end{align*}
  The indices $j$ are sampled uniformly, so the maximum time $\max_j (t -
  s_j^{(t)})$ since we last sampled the same index is an instance of the coupon
  collector's problem~\citet{WikipediaCoupon}. Because the $g_j$s are
  $\glipschitz$-Lipschitz:
  \begin{align*}
    \sigma_p^2 \le & \gamma^2 \left(m-1\right)^2 \glipschitz^2
    \expectation\left[ \max_j \norm{ w^{(t)} - w^{(s_j^{(t)})} }_2^2 \right] \\
    \le & \frac{\gamma^2 \left(m-1\right)^2 \glipschitz^2
    \wgradbound^2}{\lambda^2 T_1^2} \expectation\left[ \max_j \left( t -
    s_j^{(t)} \right)^2 \right] \\
    \le & \frac{\gamma^2 m^4 \left(1 + \left(\ln m\right)^2 +
    \nicefrac{\pi^2}{6}\right) \glipschitz^2 \wgradbound^2}{\lambda^2 T_1^2} \\
    \le & \frac{3 \gamma^2 m^4 \left(1 + \ln m\right)^2 \glipschitz^2
    \wgradbound^2}{\lambda^2 T_1^2} \eqcomma
  \end{align*}
  the second step because, between iteration $s_j^{(t)}$ and iteration $t$ we
  will perform $t - s_j^{(t)}$ updates of magnitude at most $\wgradbound /
  \lambda T_1$, and the third step because, as an instance of the coupon
  collector's problem, $\max_j (t - s_j^{(t)})$ has expectation $m H_m \le m +
  m\ln m$ ($H_m$ is the $m$th harmonic number) and variance $m^2 \pi^2 / 6$.
  Substituting into \eqref{mid:intermediate3}:
  \begin{equation*}
    \expectation\left[ \relaxedobjective\left( \bar{w}, p^{*} \right) \right] -
    f\left( w^{*} \right)
    \le \frac{\wgradbound^2 \left(2 + \ln T_1 + \ln T_2\right)}{\lambda T_2} +
    \frac{4 \gamma^2 \glipschitz^2 \ln m}{\lambda T_2} +
    \frac{3 m^4 \left(1 + \ln m\right)^2 \wgradbound^2}{2 \lambda T_1^2}
    \eqperiod
  \end{equation*}
  By the $\lambda$-strong convexity of $\relaxedobjective$:
  \begin{equation*}
    \expectation\left[ \norm{\bar{w} - w^*}_2^2 \right]
    \le \frac{2 \wgradbound^2 \left(2 + \ln T_1 + \ln T_2\right)}{\lambda^2
    T_2} +
    \frac{8 \gamma^2 \glipschitz^2 \ln m}{\lambda^2 T_2} +
    \frac{3 m^4 \left(1 + \ln m\right)^2 \wgradbound^2}{\lambda^2 T_1^2}
    \eqperiod
  \end{equation*}
  Using the facts that $\norm{\Pi_g(\bar{w}) - w^*} \le \norm{\bar{w} - w^*}$
  because $w^*$ is feasible, and that $\wgradbound = \fgradbound +
  \gamma\ggradbound$, completes the proof.
\end{prf}